\documentclass{article}

\usepackage[english]{babel}
\usepackage[utf8]{inputenc}
\usepackage{csquotes}
\usepackage{amsmath}
\usepackage{amsthm}
\usepackage{amsfonts}
\usepackage{pgfplots}
\pgfplotsset{compat=1.16}
\usepackage{tikz}

\usepackage{subcaption}

\usetikzlibrary{calc, angles, quotes}

\usepackage{algorithm}
\usepackage{algorithmic}

\newcommand{\sumgraph}{\underset{n \in [N]}{\sum}}

\def \sumgraphu {\underset{n\in [N]}{ \sum}}

\def \sumneighbors {\underset {m\in{\mathcal{N}(n)}}{\sum}}

\def \sumstrangers{\underset {m\notin{\mathcal{N}(n)}}{\sum}}

\newcommand{\flf}{\mathbf{f}_n^\top   \mathbf{L} \mathbf{f}_m}

\def \PieClam \textsc{\large P\small I\small E\large C\small L\small A\small M}
\def \pclam \textsc{\large P\large C\small L\small A\small M}
\newtheorem{theorem}{Theorem}
\newtheorem{lemma}[theorem]{Lemma}
\newtheorem{claim}[theorem]{Claim}
\newtheorem{definition}[theorem]{Definition}
\newtheorem{remark}[theorem]{Remark}

\newcommand{\ip}[2]{\left\langle#1,#2\right\rangle}

\newcommand\norm[1]{\left\lVert#1\right\rVert}

\newcommand{\ron}[1]{{{\textcolor{red}{\textbf{[Ron:} {#1}\textbf{]}}}}}

\usepackage[letterpaper,top=2cm,bottom=2cm,left=3cm,right=3cm,marginparwidth=1.75cm]{geometry}

\usepackage{amsmath}
\usepackage{graphicx}
\usepackage[colorlinks=true, allcolors=blue]{hyperref}

\title{PieClam: A Universal Graph Autoencoder Based on Overlapping Inclusive and Exclusive Communities}

\author{Daniel Zilberg and Ron Levie\\
\\
\large{Faculty of Mathematics, Technion - Israel Institute of Technology}}
\date{}

\begin{document}

\maketitle

\begin{abstract}


We propose PieClam (Prior Inclusive Exclusive Cluster Affiliation Model): a probabilistic graph model for representing any graph as overlapping  generalized communities.  Our method can be interpreted as a graph autoencoder: nodes are embedded into a code space by an algorithm that maximizes the log-likelihood of the decoded graph, given the input graph. PieClam is a community affiliation model that extends well-known methods like BigClam in two main manners. First, instead of the decoder being defined via pairwise interactions between the nodes in the code space, we also incorporate a learned prior on the distribution of nodes in the code space, turning our method into a graph generative model. Secondly, we generalize the notion of communities by allowing not only sets of nodes with strong connectivity, which we call inclusive communities, but also sets of nodes with strong disconnection, which we call exclusive communities.  To model both types of communities, we propose a new type of decoder based the Lorentz inner product, which we prove to be much more expressive than standard decoders based on standard inner products or norm distances. By introducing a new graph similarity measure, that we call the log cut distance, we show that PieClam is a universal autoencoder, able to uniformly approximately reconstruct any graph. Our method is shown to obtain competitive performance in graph anomaly detection benchmarks.

\end{abstract}


\section{Introduction}

In recent years, considerable research has concentrated on graph representation learning, aiming to develop vector representations for graph entities, including nodes, edges, and subgraphs \cite{hamilton2017representation,GRL_survet2020}. In graph autoencoders, e.g.,  \cite{kipf2016variational,grover2019graphite,JMLR:v21:19-671,mehta2019stochastic}, the vertices of a graph are embedded in a code space, where edges are inferred from the locations of the vertices in this space. Encoding graphs into a standard space has a number of advantages. While different graphs can have different sizes and topology, the code space is fixed with a fixed dimension. This helps when  learning downstream tasks, where the representation of the graph in the code space can be processed. For example, in link prediction, one infers unknown edges by defining \cite{kipf2016variational} or learning \cite{kumar2020linkreview} a function that takes pairs of nodes in the code space and predicts if there is an edge between them. In anomaly detection, one defines \cite{ding2019dominant,fan2020anomalydae} or learns \cite{chen2020gaan,ding2021aegis},  a function that takes the representation of a node and its neighborhood and predicts if this node is normal or an anomaly. In graph and node classification, the graph is represented in the code space, and one learns a model that predicts from this representation the classes of the nodes or of the graph \cite{GCN,MPNN}. 

\paragraph{Our Contribution.} In this paper, we derive a new graph autoencoder from a statistical model of graphs. In our model, similarly to SBM \cite{nowicki2001estimation,lee2019review} or community-based statistical models  \cite{airoldi2008mixed,yang2013BigClam}, graphs are generated from a combination of intersecting communities/cliques. Namely, each node belongs to a different subset of a predefined set of communities, and the affiliations of the nodes to the different communities determine the probabilities of the edges of the graph. Here, the estimation of the community affiliations is seen as an encoder of graphs to a  \emph{community affiliation space}, and the computation of the corresponding edge probabilities is seen as a decoder. As opposed to past works, we consider two types of \emph{generalized communities}. First, standard \emph{inclusive communities}, where any two nodes in the same community are likely to be connected. Second, we propose \emph{exclusive communities}, where belonging to the same community reduces the probability of nodes being connected. For illustration, consider a social network of employers/recruiters and employees/job-seekers. Such a network has roughly bipartite components, where job-seekers do not tend to connect to other seekers, and employers do not connect to other employers, but employers and seekers of the same subsector tend to connect. Hence, inclusive communities for such a graph can correspond to job titles or subsectors, and exclusive communities can correspond to sets of job-seekers and sets of employees from the same sector. 

To formalize the above ideas, we propose the \emph{Prior Inclusive Exclusive Cluster Affiliation Model (PieClam)}. This model represents graphs as overlapping inclusive and exclusive communities. Moreover, instead of embedding the nodes of a given graph in the community code space, our model also learns a prior on the code space, so new graphs can be generated from the learned model. This makes PieClam a graph generative model, like, e.g.,  \cite{kipf2016variational,grover2019graphite,JMLR:v21:19-671,mehta2019stochastic,sun2019vgraph}. 

To model both types of communities, we propose a new type of decoder based on the \emph{Lorentz inner product}, in which case the code space is typically called a \emph{pseudo Euclidean space} \cite{greub1963pseudo}. The addition of exclusive communities in our model is not merely aimed at improving it heuristically for special graphs like social networks. Rather, we prove in Theorems \ref{thm:IEUniversality} and \ref{thm:IEUniversality2} that using exclusive communities (via the Lorentz inner product) makes our model universal, namely, able to approximate with a fixed budget of parameters any graph. This is in contrast to standard decoders based on only inclusive communities, which we show are unable to represent many graphs. 

To formalize this universality property, we propose a new similarity measure between graphs with edge probabilities, that we call the \emph{log cut distance}. We formalize the universality of our model as follows: one can choose the dimension of the code space (the number of communities)  a priori, and guarantee that any graph of any size and topology can be approximated up to small log cut distance by decoding points from this fixed space. We show that other related decoders do not satisfy this universality property. 

In Section \ref{Experiments} we support our construction with experiments.
We first conduct some toy experiments that illustrate the merits of our method.
We then use PieClam to perform graph anomaly detection. Here, PieClam learns a probabilistic model corresponding to a given graph, and this model can be directly used for inspecting the probabilities of different nodes in this graph. Nodes with low probabilities are deemed to be anomalies. Our models achieve competitive performance with respect to state of the art anomaly detection methods.

Appendix \ref{Extended Related Work} offers
an extended discussion on related work.

\section{Community Affiliation Models With Prior}

Our model PieClam is best understood as an extension of the BigClam model \cite{Overlapping2012}. Hence, after introducing basic notations, we start with a detailed exposition of BigClam, before introducing our novel constructions.

\subsection{Notations}
\label{def:notations}
We denote by $\mathbb{R}$ the real line, and by $\mathbb{R}_+$ the non-negative real numbers. We denote ``and'' by $\wedge$.
We denote matrices as boldface capital letter $\mathbf{B}=\{b_{n,m}\}_{n,m}$, vectors as boldface lowercase letters $\mathbf{b}=\{b_n\}_n$,  and scalars as regular lowercase letters $b$. Vectors $\mathbf{b}\in\mathbb{R}^N$ are always column vectors, and their transpose $\mathbf{b}^{\top }$ row vectors. The rows of a matrix $\mathbf{B}\in\mathbb{R}^{N\times C}$ are denoted by the same letter in lowercase $\mathbf{b}_n^{\top }$, where $\mathbf{b}_n\in\mathbb{R}^C$, where we write in short $\mathbf{b}_n^{\top }\in\mathbb{R}^C$. A diagonal matrix $\mathbf{D}\in\mathbb{R}^{N\times N}$ with diagonal entries $\mathbf{d}$ is denoted by ${\rm diag}(\mathbf{d})={\rm diag}(d_1,\ldots, d_N)$.  The $\ell^2$ norm of a vector $\mathbf{b} \in \mathbb{R}^N$ is defined to be $\|\mathbf{b}\| = (\sum_{n=1}^Nb_n^2)^{1/2}$. 

A graph is denoted by $G=([N],E,\mathbf{A})$, where $[N]=\{1,\ldots, N\}$ is the set of $N$ nodes,  $E\subseteq [N]\times[N]$ is the set of edges, and $\mathbf{A}\in \{0,1\}^{N\times N}$ is the adjacency matrix. For weighted graphs, $\mathbf{A}\in[0,1]^{N\times N}$, where $a_{n,m}$ is the edge weight of $(n,m)$.
In this work we focus on undirected graphs, for which $(m,n)\in E \Leftrightarrow (n,m)\in E$, and $\mathbf{A}=\mathbf{A}^{\top }$.  
Any pair $(n, m)\in [N]\times[N]$ is called a \emph{dyad}.
The neighborhood of a node $n\in[N]$ is  $\mathcal{N}(n)  = \{m \in [N] \mid (m,n) \in E\}$.
 A graph-signal is a graph with node features  $G=([N],E,\mathbf{A},\mathbf{X})$ where $\mathbf{X}=\{\mathbf{x}_n^{\top }\}_{n=1}^N\in \mathbb{R}^{N\times D}$, $\mathbf{x}_n\in\mathbb{R}^D$, and $D$ is called the feature dimension. 
A random graph is a graph-valued random variable. Given a random graph with nodes $[N]$, 
 we denote the event in which  $(n,m)\in E$ by  $n\sim m$, and the event $(n,m)\notin E$ by $\neg (n\sim m)$.  
A graph is bipartite if its vertex set $[N]$ can be partitioned into two disjoint sets $\mathcal{U}$ and $\mathcal{V}$, with  $\mathcal{U}\cup\mathcal{V}=[N]$, such that every edge has one endpoint in $\mathcal{U}$ and the other in $\mathcal{V}$.

\subsection{BigClam}
\label{sec:BigClam}

 The code space in BigClam is $\mathbb{R}_+^C$, where each axis is interpreted as a community.  Each entry $f^c$ of a point $\mathbf{f}=(f^1,\ldots,f^C)\in\mathbb{R}_+^C$ is interpreted as how much the point belongs to community $c$, where $0$ means ``not included in $c$'' and $1$ ``included.'' In this paper we  call $\mathbb{R}_+^C$ \emph{the affiliation space (AS)}, and call any point in the affiliation space an \emph{affiliation feature (AF)}.

BigClam is a model where a simple random graph is decoded from a sequence of AFs $\mathbf{F}= \{\mathbf{f}_n=(f_n^1,\ldots,f_n^C)\}_{n=1}^N$ in the AS. For each pair of nodes $n,m\in[N]$, the probability of the event $\neg (n\sim m)$ (no edge between $n$ and $m$), given their amount of membership in the same community $c$, is defined to be
\[
P(\neg (n\sim m)|f^c_n, f^c_m) = e^{-f^c_n f^c_m}.
\]
BigClam makes two assumptions of independence. First, the membership in one community does not affect the membership in another community. Hence,
\begin{equation}
\label{eq:ClamEdge}
    P(\neg (n\sim m))|\mathbf{F})=P(\neg (n\sim m))|\mathbf{f}_n, \mathbf{f}_m)=e^{-\mathbf{f}_n^{\top }\mathbf{f}_m},
\end{equation}
\[P(n\sim m)|\mathbf{F})=P(n\sim m|\mathbf{f}_n, \mathbf{f}_m)=1-e^{-\mathbf{f}_n^{\top }\mathbf{f}_m}.\]
Due to this formula, BigClam is a so called \emph{Bernoulli-Poisson model} 
(see Appendix \ref{Stochastic block models}  for more details).
The second assumption is that the events of having and edge between different pairs of nodes are independent.  As a result, the probability of the entire graph $([N],E)$ conditioned on the AFs $\mathbf{F}$ is

\begingroup
\begin{equation}
\begin{aligned}
     P(E|\mathbf{F}) =
     \prod_{n \in [N]}\Bigg(
        \sqrt{\prod_{m \in \mathcal{N}(n)}P( n\sim m|\mathbf{F})\prod_{m \notin \mathcal{N}(n)} P(\neg (n\sim m)|\mathbf{F})}\Bigg)
\end{aligned}
\label{eqn:clam_prob}
\end{equation}
\endgroup
Here, the square root is taken because the product considers each edge twice. 

BigClam consist of a decoder, decoding from the AFs $\mathbf{F}$ the random graph $G(\mathbf{F})$ with node set $[N]$ and independent Bernoulli distributed edges with probabilities $\mathbf{P}=\{P(n\sim m|\mathbf{f}_n, \mathbf{f}_m)\}_{n,m=1}^N$.

From the encoding side, given a simple graph $([N],E)$, BigClam encodes the graph into the affiliation space by maximizing the log likelihood with respect to the AFs $\mathbf{F}$
\begingroup
\begin{align}
\label{eqn:naiveloss}
     l(\mathbf{F}) &= \frac{1}{2}\sumgraph \Big(\sumneighbors{}{\log(1-e^{-\mathbf{f}_n^{\top } \mathbf{f}_m})} - 
 \sumstrangers{\mathbf{f}_n^{\top } \mathbf{f}_m}\Big).
 \end{align}
 \endgroup
BigClam is optimized by gradient descent, with update at iteration $i$
    \begin{align}
    \label{eqn:vanilla_optimization}
    \mathbf{F}^{(i+1)} = \mathbf{F}^{(i)} + \delta\nabla_{\mathbf{F}}l({\mathbf{F}}^{(i)}),
    \end{align}
for some learning rate $\delta>0$.
In order to implement the above iteration with $O(|E|)$ operations at each step, instead of $O(N^2)$, the loss can be rearranged as
\begin{equation}
\begin{aligned}
\label{eqn:comp_loss}
 2l(\mathbf{F})=
    \sumgraphu \Big(\sumneighbors{\log(e^{\mathbf{f}_n^{\top } \mathbf{f}_n} - 1)}  - \mathbf{f}_n^\top {\sumgraph{\mathbf{f}_m}} + \|\mathbf{f}_n\|^2\Big).
\end{aligned}
\end{equation}
The gradient of the loss is now
\begin{align}
\label{eqn:BCupdate}
    \nabla_{\mathbf{f}_n}l = 
    \sumneighbors{\mathbf{f}_m \big(1-e^{-\mathbf{f}_n^\top  \mathbf{f}_m}
    \big)^{-1}} - \sumgraph{\mathbf{f}_m} + \mathbf{f}_n .
\end{align}
Since the global term needs to be calculated only once, the number of operations is  $O(|E|)$ instead of $O(N^2)$.  

We observe that the optimization process is a message passing scheme.
Looking at the dynamics of the optimization process, we see that every node is pushed in the direction of a weighted average of all of its neighbors, and pushed in the opposite direction by an average of all of the nodes. The sum of both forces tends to drive communities towards the axes in the optimization dynamics.

\subsection{Inclusive-Exclusive Cluster Affiliation Model}

The BigClam decoder has a limitation due to a  ``triangle inequality type'' behavior. Namely, suppose that we would like to construct two features $\mathbf{f}_1$ and $\mathbf{f_2}$ with strong connectivity to a third feature $\mathbf{f}_3$, and  we are limited by a fixed affiliation feature dimension $C$. A naive approach to achieve this would be to put $\mathbf{f}_1$ and $\mathbf{f_2}$ close to $\mathbf{f}_3$ so they have large inner products $\mathbf{f}_1^{\top }\mathbf{f}_3$ and $\mathbf{f}_2^{\top }\mathbf{f}_3$ . This would mean that $\mathbf{f}_1^{\top }\mathbf{f}_2$ would also be large, so $\mathbf{f}_1$ would be strongly connected to  $\mathbf{f}_2$ under the BigClam model. However, some graphs, like bipartite graphs, do not exhibit this triangle inequality type behavior. For a rigorous treatment, see Section \ref{BigClam is Not Universal} and Appendix \ref{Proof That BigClam Is Not Universal}.  Next, we build the IeClam decoder, that allows decoding any graph, that may have bipartite components without being limited by a  triangle inequality-type behavior. 

Our \emph{Inclusive Exclusive Cluster Affiliation Model (IeClam)} can be extended from BigClam by replacing the inner product in the non-edge probability (\ref{eq:ClamEdge}) by the more expressive \emph{Lorentz inner product}, which does not enforce a triangle inequality-type behavior. 
For that, we extend the affiliation space (AS) to be $\mathbb{R}^{2C}$, with two types of communities. The first $C$ axes are called the \emph{inclusive communities}, and their corresponding features are called \emph{inclusive affiliation features (IAF)}, denoted by $\mathbf{t}\in\mathbb{R}^C$. The last $C$ axes are called the \emph{exclusive communities}, with \emph{exclusive affiliation features (EAF)}, denoted by $\mathbf{s} \in \mathbb{R}^C$.\footnote{More generally, one can define a different number of inclusive and exclusive communities.} We define the concatenated affiliation feature by $\mathbf{f} = (\mathbf{t},\mathbf{s})\in\mathbb{R}^{2C}$. Given a sequence of affiliation features  $\mathbf{F}=\{\mathbf{f}_n\}_{n=1}^N\in\mathbb{R}^{N\times 2C}$, IeClam defines the probability of a single edge by
\begin{equation}
\begin{split}
    P(n\sim m| \mathbf{f}_n,\mathbf{f}_m) & = 1-\exp\big({-\mathbf{t}_n^{\top } \mathbf{t}_m +  \mathbf{s}_n^{\top } \mathbf{s}_m}\big) \\  
   &  = 1-\exp\big({-\mathbf{f}_n^{\top } \mathbf{L} \mathbf{f}_m}\big),
    \label{eq:LorFeature}
    \end{split}
\end{equation}
where $\mathbf{L} = {\rm diag}(1,...1,-1,...-1)$. The bilinear form $(\mathbf{u},\mathbf{v})\mapsto \mathbf{u}^\top  \mathbf{L} \mathbf{v}$ is called the \emph{Lorentz inner product}, and has its roots in special relativity. See Appendix \ref{The Lorentz Inner Product} for more details on the Lorentz inner product. 

Note that $\mathbf{L}$ is not positive-definite, so it does not actually define an inner product. Moreover, $\mathbf{f}_n^\top  \mathbf{L} \mathbf{f}_m$ can be negative even if $\mathbf{f}_n,\mathbf{f}_m\in \mathbb{R}_+^C$. To guarantee that (\ref{eq:LorFeature}) defines a proper probability between $0$ and $1$, we limit the affiliation space as follows.  
\begin{definition}
  A \emph{cone of non-negativity} is a  subset $\mathcal{C}$ of $\mathbb{R}^{2C}$ such that for every $\mathbf{f},\mathbf{g}\in \mathcal{C}$ we have 
$\mathbf{f}^\top  \mathbf{L} \mathbf{g} \geq 0$.  
\end{definition}
If we limit the affiliation space to be a cone of non-negativity, then IeClam gives well defined probabilities in $[0,1]$.
In our experiments, we restrict ourselves to the following simple construction of a cone of non-negativity, noting that it is not the only possible construction. 
%
\begin{definition}
 The \emph{pairwise cone} $\mathcal{T}$ is defined to be the set of affiliation features $\mathbf{f}=(\mathbf{t},\mathbf{s})\in\mathbb{R}^{2C}$ such that for every $c\in[C]$ we have $-t^c \leq s^c\leq t^c$.    
\end{definition}
It is easy to see that $\mathcal{T}$ is a cone of non-negativity. Indeed, for any $\mathbf{f}_1=(\mathbf{t}_1,\mathbf{s}_1), \mathbf{f}_2=(\mathbf{t}_2,\mathbf{s}_2) \in \mathcal{T}$,  we have  
 \[\forall c\in[C]: \quad  t_1^c t_2^c - s_1^c s_2^c \geq 0,\] so
\[\mathbf{f}_1^\top  \mathbf{L} \mathbf{f}_2 = \sum_{c=1}^C t_1^c t_2^c - s_1^c s_2^c \geq 0.\] 
As opposed to the BigClam decoder, the IeClam model can approximate a bipartite graph with a small number of communities. Namely, a bipartite graph with two (disjoint) sides $\mathcal{U},\mathcal{V}\subset[N]$ and constant probability for edges between $\mathcal{U}$ and $\mathcal{V}$ can be represented using $C=1$. Here,  all of the nodes in $\mathcal{U}$ are encoded to $(a,a)$,  and all of the nodes of $\mathcal{V}$ are encoded to $(a,-a)$, for some $a\geq 0$. This gives zero probability for edges within each part, and probability $1- e^{-2a^2}$ for edges between the parts. 

\begin{remark}
    The above construction gives an interpretation for each pair of axes $(t^c,s^c)$ in $\mathcal{T}$ as a \emph{generalized community} which can model anything between a clique and a bipartite component.
\end{remark}

Given AFs $\mathbf{F}$ in a cone, IeClam is seen as a decoder, by decoding $\mathbf{F}$ into the random graph $G(\mathbf{F})$ with node set $[N]$ and independent Bernoulli distributed edges with probabilities $\mathbf{P}=\{P(n\sim m|\mathbf{f}_n, \mathbf{f}_m)\}_{n,m=1}^N$.

For affiliation features $\mathbf{F}\in \mathcal{T}^N$, the probability of the graph $([N],E)$ given $\mathbf{F}$ of IeClam is
\begin{equation}
\begin{aligned}
   &  P(E|\mathbf{F}) \\
   & = \prod_{n \in [N]}\Bigg(
        \sqrt{\prod_{m \in \mathcal{N}(n)}(1-e^{-\mathbf{f}_n^\top  \mathbf{L} \mathbf{f}_m})\prod_{m \notin \mathcal{N}(n)} e^{-\mathbf{f}_n^\top  \mathbf{L} \mathbf{f}_m}}\Bigg).
\end{aligned}
\label{eqn:IeClam_prob}
\end{equation}
Like BigClam, IeClam is optimized by maximizing the log likelihood with gradient descent
\begin{equation}
\begin{aligned}
   &  2 l(\mathbf{F}) =
    \\ &\sumgraphu \Big(\sumneighbors{}\log(1 - e^{-{\mathbf{f}_n}^\top  \mathbf{L} \mathbf{f}_m})  -   \sumstrangers{}  \mathbf{f}_n ^\top  \mathbf{L} \mathbf{f}_m\Big)
    \end{aligned}
\label{eqn:IeClam_loss}
\end{equation}
This loss can be efficiently implemented on sparse graph by the formulation
\begingroup
\begin{equation}
\begin{aligned}
\label{eqn:loss_reparam_ie}
   &  2l(\mathbf{F})=\\
   & \sumgraphu \Big(\sumneighbors{\log(e^{\flf} - 1)}  - \mathbf{f}_n^\top  \mathbf{L}{\sumgraph{\mathbf{f}_m}} + \mathbf{f}_n^\top  \mathbf{L}\mathbf{f}_n\Big).
\end{aligned}
\end{equation}
\endgroup
The gradient of the loss for node $n$ is
\begin{align}
\label{eqn:BCupdate_ie}
    \nabla_{\mathbf{f}_n}l = 
   \mathbf{L}\bigg( \sumneighbors{\mathbf{f}_m \big(1-e^{-\flf}
    \big)^{-1}} - \sumgraph{\mathbf{f}_m} + \mathbf{f}_n \bigg).
\end{align}
Notice that all of the calculations are the same as BigClam, up to replacing the dot product by the Lorenz inner product.

\subsection{Community Affiliation Models With Prior}

BigClam and IeClam are not  generative graph models. Indeed, these methods only fit a conditional probability of the graph, conditioned on the AF values, but the methods do not learn the probability of the AFs over the affiliation space. Hence, the total probability of $E\wedge\mathbf{F}$ is not defined. 
To extend IeClam (and similarly BigClam) into probabilistic generative models, we define a prior probability distribution over the affiliation cone space $\mathcal{C}\subset \mathbb{R}^{2C}$, with probability density function $p:\mathbb{R}^{2C}\rightarrow [0,\infty)$ supported on $\mathcal{C}$. Using Bayes law, we now obtain the joint probability $P(E \wedge \mathbf{F})$ of the edges and community affiliation features  via the probability density function
\begin{align*}
    p(E ,\mathbf{F}) = P(E|\mathbf{F}) p(\mathbf{F}).
\end{align*}
We assume that the prior probabilities of all nodes are independent, namely,  
\[p(\mathbf{F}) = \underset{n\in [N]}{\prod} p(\mathbf{f}_n).\]
 Hence, the probability densities that a dyad $(n,m)$ is an edge or non-edge are
\[p(n\sim m, \mathbf{f_n}, \mathbf{f_m}) = p(\mathbf{f}_n)p(\mathbf{f}_m)(1 - e^{-{\mathbf{f}_n}^\top  \mathbf{L} {\mathbf{f}_m}}),\]
\[p(\neg(n\sim m), \mathbf{f_n}, \mathbf{f_m}) = p(\mathbf{f}_n)p(\mathbf{f}_m)e^{-{\mathbf{f}_n}^\top  \mathbf{L} {\mathbf{f}_m}}.\]
As before, we assume that the probabilities of different edges are independent, which gives
\begin{equation}
\begin{aligned}
\label{eqn:likelihood_prior}
    p(E, \mathbf{F}) =  \prod_{n \in [N]} p(\mathbf{f}_n)
        \sqrt{\prod_{m \in \mathcal{N}(n)}P(n\sim m|\mathbf{F}_m)\prod_{m \notin \mathcal{N}(n)}{P(\neg(n\sim m)|\mathbf{f}_n, \mathbf{f}_m)}
        }. 
 \end{aligned}
 \end{equation}
Now, the log likelihood loss is

\begin{equation}
\begin{aligned}
\label{eqn:loss_reparam_pie}
    l(\mathbf{F})= \sumgraphu \Bigg(\log(p(\mathbf{f}_n)) +
    \quad  \frac{1}{2} \Big(\sumneighbors{\log(e^{\flf} - 1)}  -  \mathbf{f}_n^\top  \mathbf{L}{\sumgraph{\mathbf{f}_m}} + \mathbf{f}_n^\top  \mathbf{L}\mathbf{f}_n \Big) \Bigg).
\end{aligned}
\end{equation}

We call this extension of IeClam PieClam (Prior Inclusive Exclusive Cluster Affiliation Model). We similarly extend BigClam to PClam (Prior Cluster Affiliation Model) by replacing $\mathbf{L}$ in (\ref{eqn:loss_reparam_pie}) with the identity matrix. 

Observe that the PieClam loss is similar to the IeClam loss only, with the addition of the prior, acting as a per node regularization term. The prior attracts all nodes to areas of higher probability during the optimization dynamics.

In order to sample from the above generative models, we first sample features $\{\mathbf{f}_n\}_{n\in [N]}$ according to $p$, and then connect them using either the BigClam or IeClam conditional probability.
To model the prior in practice, we use \emph{realNVP},  which is a \emph{normalizing flow} neural network model \cite{dinh2016density}. For more details on normalizing flows, see Appendix \ref{extended_sec:normflows}.

\paragraph{PieClam for graphs with node features.}

So far, we have used only the topology of the graph, not considering node features. 
We extend PieClam (and PClam) to graph-signals $([N], E, \mathbf{X})$ as follows. Concatenate the feature space of $\mathbf{X}$ with the affiliation space, and learn the prior on this combined space. This only affects the prior $p$. The conditional edge probabilities are  defined only in terms of the affiliation features, as before.

\section{Universality of PieClam and IeClam}

In this section, we define the universality of graph autoencoders, and prove that IeClam and PieClam are universal, while BigClam and PClam are not. The motivation behind the universality definition is that we would like to uniformly choose the dimension of the code space, such that every graph can be approximated by decoding some points in this fixed code space. Namely, we would like one \emph{universal} decoder that works for all graph, as opposed to choosing the dimension of the code space depending on the graph.

\subsection{General Graph Autoencoders}

Next, we define a general decoder that defines edge probabilities  by operating on pairs of points in a code space. 

\begin{definition}
   A \emph{pairwaise decoder} over the code spaces $\mathbb{R}^M$ is a mapping $D_M:\mathbb{R}^{2M}\rightarrow[0,1]$. Given $N$ points in the code space $\mathbf{z}=\{z_n\in\mathbb{R}^M\}_{n=1}^N$, the decoded graph $G_N(\mathbf{z})$ is the weighted graph with adjacency matrix  
\[\mathbf{D}_M(\mathbf{z})=\big(D_M(z_n,z_k)\big)_{n,k=1}^N.\]
\end{definition}

Clam models are special cases of pairwise decoders.

\subsection{Log Cut Distance}

Our definition of universality has the following form: for every error tolerance $\epsilon>0$, there is a choice of the dimension $M(\epsilon)$ of the code space such that every graph can be approximated up to error $\epsilon$ by decoding some points in this space. To formalize the ``up to error $\epsilon$'' statement, we present in this section a new graph similarity measure which we call the log cut distance.
Our construction is based on a well-known graph similarity measure called the cut norm.

\begin{definition}
  The \emph{cut norm} of a matrix $\mathbf{X}\in\mathbb{R}^{N\times N}$ is defined to be
    \begin{equation}
    \label{eq:GraphCutNorm0}  \|\mathbf{X}\|_{\square}:=\frac{1}{N^2}\sup_{\mathcal{U},\mathcal{V}\subset[N]}\Big|\sum_{i\in \mathcal{U}}\sum_{j\in \mathcal{V}}x_{i,j}\Big|.
\end{equation}  
\end{definition}
The \emph{cut metric} $\|\mathbf{A}-\mathbf{B} \|_{\square}$ between two adjacency matrices $\mathbf{A}$ and $\mathbf{B}$ is interpreted as the difference between the edge densities of $\mathbf{A}$ and $\mathbf{B}$ on the block $\mathcal{U}\times\mathcal{V}$ on which their edge densities are the most different.

The following graph similarity measure modifies the cut norm, making it appropriate for graphs with random edges over a fixed node set.
\begin{definition}
\label{def:KL1}
Given two random graphs over the nodes set $[N]$, with independently Bernoulli distributed edges, with probabilities $\mathbf{P}=\{p_{n,m}\}_{n,m\in[N]}$ and $\mathbf{Q}=\{q_{n,m}\}_{n,m\in[N]}$ respectively, their \emph{log cut distance} is defined to be 
\begin{equation}
    \begin{split}
        D_{\square}(\mathbf{P} || \mathbf{Q}):=  \inf_{0< e,d\leq 1} \Bigg( e+d+ \frac{1}{N^2}\sup_{\mathcal{U},\mathcal{V}\subset[N]}\Big|\log\Big(\prod_{n\in \mathcal{U}}\prod_{m\in \mathcal{V}}
\frac{1-(1-e)p_{n,m}}{1-(1-d)q_{n,m}}\Big)\Big|\Bigg).
    \end{split}
    \label{eq:KL1}
\end{equation}
\end{definition}
The second term in (\ref{eq:KL1}) is the cut distance $\| \tilde{\mathbf{P}}-\tilde{\mathbf{Q}}\|_{\square}$ between the matrix $\tilde{\mathbf{P}}$ with entries 
\[\tilde{p}_{n,m}=-\log(1-(1-e)p_{n,m})\]
and the matrix $\tilde{\mathbf{Q}}$
 with entries 
 \[\tilde{q}_{n,m}=-\log(1-(1-d)q_{n,m}).\] Namely, the cut distance between the log likelihoods of non-edges. The parameters $e$ and $d$ make the $[0,1]$-valued probabilities valid inputs to the log. The goal of $e,d$ is to regularize the probability of the edges, where higher regularization is penalized via the additive term $e+d$ in (\ref{eq:KL1}).

For each choice of a cut $\mathcal{U},\mathcal{V}\subset[N]$,  the term
 \begin{equation}
 \label{eq:KL2}
     \frac{1}{N^2}\log\Big(\prod_{n\in \mathcal{U}}\prod_{m\in \mathcal{V}}
\frac{1-(1-e)p_{n,m}}{1-(1-d)q_{n,m}}\Big)
 \end{equation}
is somewhat similar in structure to an un-normalized KL divergence, or distance of log likelihoods, between the non-edge probabilies of the graphs $\mathbf{P}$ and $\mathbf{Q}$ over the dyads between  $\mathcal{U}$ and $\mathcal{V}$. Here, ``un-normalized'' means that the dyads are drawn  uniformly with probabilities $1/N^2$, but the sum of probabilities is $|\mathcal{U}|\cdot|\mathcal{V}|/N^2$ and not 1. The un-normalized uniform distribution discourages the supremum inside the definition of $D_{\square}$ from choosing small blocks for maximizing (\ref{eq:KL2}). Note that normalized uniform distributions would lead $D_{\square}$ to choose  small blocks, which do not reflect meaningful empirical estimates of the edge statistics (the edge densities of small blocks would not be interpretable as expected number of edges). To conclude, $D_{\square}(\mathbf{P} || \mathbf{Q})$ is interpreted as the maximal divergence between $\mathbf{P}$ and $\mathbf{Q}$ over all blocks, up to the best regularizers $e,d$.

In our analysis we compute the log cut distance between the random decoded graph $\mathbf{P}$ and the deterministic target graph $\mathbf{A}$. While $\mathbf{A}$ has edge probabilities in $\{0,1\}$, $\mathbf{P}$ has edge probabilities in $[0,1)$ for Clam models. Therefore, we only require regularization  for $\mathbf{A}$. We hence consider the following modified version of Definition \ref{def:KL1}.
\begin{definition}
\label{def:KL2}
Given an unweighted graph with adjacency matrix $\mathbf{A}\in\{0,1\}^{N\times N}$ and a random graph over the nodes set $[N]$, with independently Bernoulli distributed edges with probabilities $\mathbf{P}=\{0\leq p_{n,m}<1\}_{n,m\in[N]}$, the log cut distance between $\mathbf{P}$ and $\mathbf{A}$ is defined to be
    \[
    D_{\square}(\mathbf{P} || \mathbf{A}):= 
 \inf_{0< d\leq1} \Bigg( d+ \frac{1}{N^2}\sup_{\mathcal{U},\mathcal{V}\subset[N]}\Big|\log\Big(\prod_{n\in \mathcal{U}}\prod_{m\in \mathcal{V}}
\frac{1-p_{n,m}}{1-(1-d)a_{n,m}}\Big)\Big|\Bigg).
    \]
\end{definition}

Lastly, in case both $\mathbf{P}$ and $\mathbf{Q}$ are $[0,1)$-valued, a simple version of the log cut distance is (\ref{eq:KL1}) with the choice $e=d=0$, namely,
 \begin{equation}
\label{eq:KL3}
 D^0_{\square}(\mathbf{P} || \mathbf{Q}):=
         \frac{1}{N^2}\sup_{\mathcal{U},\mathcal{V}\subset[N]}\Big|\log\Big(\prod_{n\in \mathcal{U}}\prod_{m\in \mathcal{V}}
\frac{1-p_{n,m}}{1-q_{n,m}}\Big)\Big|.
    \end{equation}

\subsection{Universal Graph Autoencoders}

We are now ready to define the universality of general pairwise autoencoders. Motivated by the fact that a Clam autoencoder is actually a family of autoencoders, parameterized by the number of communities, we also define general pairwise decoders as families.

\begin{definition}
    A family of code spaces $\mathbb{R}^M$ and corresponding pairwaise decoders $D_M:\mathbb{R}^{2M}\rightarrow[0,1]$, parametrized by $M\in\mathbb{N}$, is called  \emph{universal} if for every $\epsilon>0$ there is $M\in\mathbb{N}$ (which depends only on $\epsilon$) such that for every $N\in\mathbb{N}$ and every graph  with adjacency matrix $\mathbf{A}$ and $N$ nodes there are $N$ points in the code space $\{z_n\in\mathbb{R}^M\}_{n=1}^N$ such that 
     \[D_{\square}(\mathbf{D}_M(\mathbf{z}) || \mathbf{A}) <\epsilon.\] 
\end{definition}

\subsection{BigClam and PClam are Not Universal}
\label{BigClam is Not Universal}

We now show that BigClam (and hence also PClam) is not a universal autoencoder since it cannot approximate bipartite graphs.
Consider the bipartite graph $\mathbf{B}$ with $N$ nodes at each part, and probability $1-e^{-a^2}$ for an edge between the two parts, and $0$ within each part. 
Since in this case all probabilities are less than 1, we can use (\ref{eq:KL3}) as the definition of the log cut distance. The analysis for Definition \ref{def:KL2} extends naturally.

Let $\mathbf{P}$ be a decoded BigClam graph. 
Our goal is to show that there is no way to make $D^0_{\square}(\mathbf{P} || \mathbf{Q})$ small by choosing the dimension $C$ uniformly with respect to $N$. In fact, we will show BigClam cannot approximate a bipartite graph at all.\footnote{In Appendix \ref{BigClam With No Self Loops Approximating Bipartites} we show that one can approximate a bipartite graph of $2N$ nodes using $C=N^2$ classes in BigClam if the model ignores self-loops.} 

\begin{claim}
\label{claim:Clam_bi}
    Under the above construction, 
    \[D^0_{\square}(\mathbf{P} || \mathbf{B}) \geq \frac{a^2}{16}.\]
    As a result, BigClam is not a universal autoencoder.
\end{claim}
The proof is given in Appendix \ref{Proof That BigClam Is Not Universal}.
We note that one can similarly show that BigClam is not universal also with respect to the log cut distance of Definition \ref{def:KL2}.

\subsection{Universality of IeClam and PieClam}
We are now ready to show that IeClam (and hence also PieClam) is a universal autoencoder. The proofs of the following two theorems are in Appendix \ref{Proof of the Universality of PieClam} and \ref{Proof of the Universality of IeClam in a Cone of Non-negativity}.

We give two versions for the universality result. The first is without a cone restriction, and requires a relatively small number of communities for the given error tolerance. The corresponding decoder produces edge weights that can be negative. The second theorem is restricted to the pairwise cone of non-negativity, and has pessimistic  asymptotics for the required number of communities given an error tolerance. This decoder is guaranteed to produce proper edge probabilities in $[0,1)$.

\begin{theorem}
\label{thm:IEUniversality}
    For every epsilon $\epsilon>0$, every $N\in\mathbb{N}$, and every adjacency matrix $\mathbf{A}\in[0,1]^{N\times N}$, there are $N$ affiliation features $\mathbf{F}\in \mathbb{R}^{2K}$ of dimension  $K=-9\log(\epsilon/2)^2/\epsilon^2$ such that the corresponding IeClam model $\mathbf{P}=\{P(n\sim m|\mathbf{f}_n, \mathbf{f}_m)\}_{n,m=1}^N$ satisfies 
     \[[D_{\square}(\mathbf{P} || \mathbf{A}) <\epsilon.\] 
    Here, the log cut distance is from Definition \ref{def:KL2}.
    As a result, 
    IeClam and PieClam are universal autoencoders with code space $\mathbb{R}^{2K}$.
\end{theorem}

\begin{theorem}
\label{thm:IEUniversality2}
    For every epsilon $\epsilon>0$, every $N\in\mathbb{N}$, and every adjacency matrix $\mathbf{A}\in[0,1]^{N\times N}$, there are $N$ affiliation features $\mathbf{F}$ in the cone of pairwise non-negativity  $\mathcal{T}\subset \mathbb{R}^{2C}$ of dimension  $C=2^{4\lceil -\log(\epsilon/2)^2/\epsilon^2\rceil}$  such that the corresponding IeClam model $\mathbf{P}=\{P(n\sim m|\mathbf{f}_n, \mathbf{f}_m)\}_{n,m=1}^N$ satisfies 
     \[D_{\square}(\mathbf{P} || \mathbf{A}) <\epsilon.\] 
    Here, the log cut distance is from Definition \ref{def:KL2}.
    As a result, 
    IeClam and PieClam are universal autoencoders with code space $\mathcal{T}$.
\end{theorem}

\section{Experiments}
\label{Experiments}

\subsection{Reconstructing Synthetic Priors}
\label{sec:reconstruction_prior}

We consider a ground-truth synthetic prior $p:\mathbb{R}^C\rightarrow[0,\infty)$ in PClam. 
We sample $N=500$ points from the prior, decode the corresponding PClam graph, and sample a simple graph from the random Bernoulli edges. Then, given these sampled graph, we fit to it a PClam and  model. In Figure~\ref{fig:prior_reconstruction} we compare the ground-truth prior to the reconstructed prior. We observe that even though our method is unsupervised, it still manages to capture the prior qualitatively well. More details are given in Appendix \ref{Extended Details on Experiments}.

\begin{figure}[H]
    \begin{minipage}{0.3\columnwidth}
        \includegraphics[width=\textwidth]{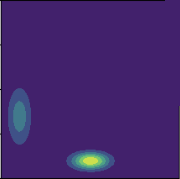}
        
    \end{minipage}\hfill
    \begin{minipage}{0.3\columnwidth}
        \includegraphics[width=\textwidth]{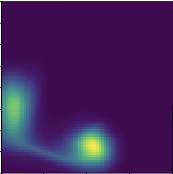}
        
    \end{minipage}\hfill
    \begin{minipage}{0.3\columnwidth}
        \includegraphics[width=\textwidth]{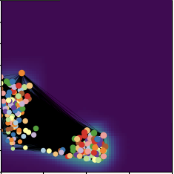}
        
    \end{minipage}
    \caption{Left to right: Synthetic prior in an affiliation space of two inclusive communities.  Reconstructed prior by PClam with normalizing flow.  Reconstructed affiliation features by PClam.}
    \label{fig:prior_reconstruction}
\end{figure}


\subsection{Reconstructing Synthetic SBMs}
\label{sec:reconstruction_sbm}

In Figures \ref{fig:sbm_PieClam} and \ref{fig:sbm_pclam} we consider a synthetic SBM, and sample simple graph with $N=210$ nodes from it.  We then fit A PClam and PieClam model to it. 
The SBM is not dominant diagonal, so it cannot be well approximated by the PClam model (which finds a nested community structure), while the PieClam model approximates it well qualitatively. 
  Additional details are given in Appendix \ref{Extended Details on Experiments}.

\begin{figure}[H]
    \begin{minipage}{0.29\columnwidth}
        \includegraphics[width=\textwidth]{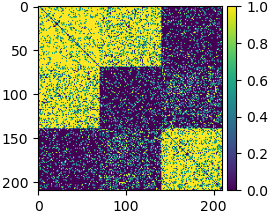}
        
    \end{minipage}\hfill
    \begin{minipage}{0.29\columnwidth}
        \includegraphics[width=\textwidth]{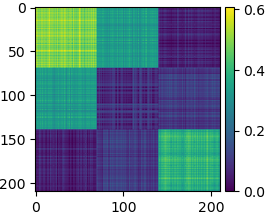}
        
    \end{minipage}\hfill
    \begin{minipage}{0.2\columnwidth}
        \includegraphics[width=\textwidth]{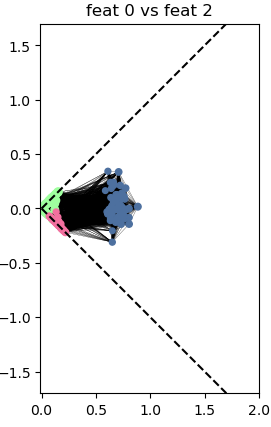}
        
    \end{minipage}
    \begin{minipage}{0.2\columnwidth}
        \includegraphics[width=\textwidth]{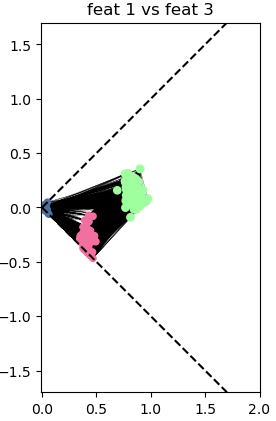}
        
    \end{minipage}
    \caption{Left to right:  Adjacency matrix sampled from SBM with three classes and 9 blocks.  Adjacency matrix of the fitted PieClam graph, with two inclusive and two exclusive communities.  Affiliation features of the PieClam matrix in $\mathcal{T}$ projected to $(t^1,s^1)$ and $(t^2,s^2)$.}
    \label{fig:sbm_PieClam}
\end{figure}

\begin{figure}[H]
    \begin{minipage}{0.3\columnwidth}
        \includegraphics[width=\textwidth]{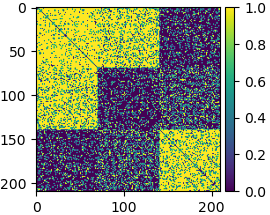}
        
    \end{minipage}\hfill
    \begin{minipage}{0.3\columnwidth}
        \includegraphics[width=\textwidth]{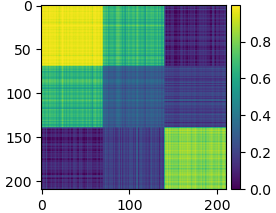}
        
    \end{minipage}\hfill
    \begin{minipage}{0.3\columnwidth}
        \includegraphics[width=\textwidth]{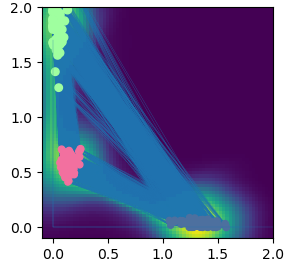}
        
    \end{minipage}
    \caption{Left to right: Adjacency matrix sampled from SBM. Fitted  PClam adjacency matrix based on two inclusive communities.  Affiliations features of the PCLam matrix.}
    \label{fig:sbm_pclam}
\end{figure}

\subsection{Anomaly Detection}

In unsupervised node anomaly detection, one is given a graph with node features, where some of the nodes are unknown anomalies. The goal is to detect these anomalous nodes, without supervising on any example of normal or anomalous nodes, by only using the structure of the graph and the node features.  
We use Clam models for node anomaly detection by fitting the Clam model to the graph and flagging nodes as anomalous if they satisfy the following different criteria.
\begin{itemize}
    \item 
    \textbf{(S)} \textbf{Star probability:} Given any Clam model, a node $n$ is called anomalous if $\prod_{m\in\mathcal{N}(n)}P(n\sim m| \mathbf{F})<\delta$. 
    \item 
    \textbf{(P)} \textbf{Prior probability:} Given any Clam model with prior, a node $n$ is called anomalous if $p(\mathbf{f}_n)<\delta$. 
    \item 
     \textbf{(PS)} \textbf{Prior star probability:} Given any Clam model with prior, a node $n$ is called anomalous if $p(\mathbf{f}_n)\prod_{m\in\mathcal{N}(n)}P(n\sim m| \mathbf{F})<\delta$. 
\end{itemize}
 We reduce the dimension of the node features of the input graphs to $100$ using truncated SVD, unless the dimension of the features is smaller than $100$ in which case we only normalize them to have zero mean and standard deviation one.   We use an affiliation space embedding dimension of 30 for $\mathbf{F}$ for PieClam and IeClam, and $24$ for BigClam.
Every Clam method starts with a random embedding $\mathbf{F}$.  Clam models with prior are trained with the following steps. 
 ($\mathbf{F}$-$t$): given a fixed prior $p$, optimize only $\mathbf{F}$ for $t$ steps. ($p$-$t$): given a fixed embedding $\mathbf{F}$, optimize only $p$ for $t$ steps. For regularization, in each iteration of of $p$ we add Gaussian noise with amplitude 0.01 to the affiliation features. We train PieClam with scheme $\mathbf{F}$-500 $\rightarrow$ $p$-1300 $\rightarrow$ $\mathbf{F}$-500 $p$-1300 with learning rate of $2e^{-6}$ on $\mathbf{F}$  and $1e^{-6}$ on $p$.  
For the models which only optimize $\mathbf{F}$ (IeClam and BigClam)  we use the following configurations: We train IeClam on 2500 iterations with learning rate 1e-6 using 30 communities. we train BIGClam on 2200 iterations with learning rate 1e-6.
More details on hyper-parameters are given in Appendix \ref{Extended Details on Experiments}.

In Table~\ref{Table1} we compare the performance of Clam methods to DOMINANT \cite{ding2019dominant}, AnomalyDAE \cite{fan2020anomalydae}, OCGNN \cite{wang2021ocgnn}, AEGIS \cite{ding2021aegis}, GAAN \cite{chen2020gaan} and TAM \cite{qiao2024tam} on the datasets Reddit, Elliptic, and Photo. The hyper-parameters of the competing methods are taken as the recommended values from the respective papers. The results are taken from Table 1 in \cite{qiao2024ggad}.
We observe that our methods are first and second place on all datasets. Moreover, S- IeClam, PS-PieClam and S BigClam each beats the competing methods in two out of the three datasets.        

\begin{table}[H]
\centering
\begin{tabular}{|c|c|c|c|}
    \hline
  Method  & Reddit& Elliptic& Photo\\
    \hline

 (S)- IeClam& \textbf{0.639}& 0.440&\textbf{0.592}\\
 (S) - PieClam& 0.610& 0.435&0.556\\
 (P) - PieClam& 0.567& \textbf{0.610}&0.425\\
 (PS) - PieClam& *0.612& \underline{0.551}&0.507\\
 (S) - BigClam & \underline{0.637}& 0.434&\underline{0.581}\\
 \hline
 DOMINANT  & 0.511& 0.296&0.514\\
 AnomalyDAE &0.509& *0.496&0.507\\
 OCGNN  &0.525 & 0.258&0.531\\
 AEGIS &0.535 & 0.455&0.552\\
 GAAN &0.522 &0.259 &0.430\\
 TAM &0.606 & 0.404 & *0.568\\
 \hline
\end{tabular}
\caption{Comparison of Clam anomaly detectors with competing methods. 
 First place in \textbf{boldface}, second with \underline{underline}, third with *star. We observe that our methods are first and second place on all datasets. Moreover, S- IeClam, PS-PieClam and S BigClam each beats the competing methods in two out of the three datasets. The accuracy metric is areas under curve (AUC).}
\label{Table1}
\end{table}

\section{Conclusion}

We introduced PieClam, a new probabilistic graph generative model. PieClam models graphs via embedding the nodes into an inclusive and exclusive communities space, learning a prior distribution in this space, and decoding pairs of points in this space to edge probabilities, such that points are more likely to be connected the more inclusive communities and the less exclusive communities they share. We showed that PieClam is a universal autoencoder, able to approximate any graph, where the budget of parameters (the number of communities) can be predefined, irrespective of any property of specific graph, not even the number of nodes. Our experiments show that PieClam achieves competitive results when used in graph anomaly detection.

One limitation of PieClam is that, for attributed graphs, it only models the node features through the prior in the community affiliation space, but not via the conditional probabilities of the edges (given the community affiliations). Future work will deal with extending PieClam to also include the node (or edge) features in the edge conditional probabilities. 
Another limitation of our analysis is that the log cut distance is mainly appropriate for dense graphs. Future work will extend this metric to sparse graphs. This can be done, e.g., similarly to the sparse constructions in \cite{Levie_intersecting_reg_24}.

\section*{Acknowledgements}
This research was supported by the Israel Science Foundation (grant No. 1937/23)

\bibliographystyle{plain}

\appendix

\begin{center}
{\huge \textbf{Appendix}}

\end{center}

\section{Proofs}

\subsection{Proof That BigClam Is Not Universal}
\label{Proof That BigClam Is Not Universal}
In this subsection we Prove Claim \ref{claim:Clam_bi}. 
Consider the bipartite graph $\mathbf{B}$ with $N$ nodes at each part, and probability $1-e^{-a^2}$ for an edge between the two parts, and $0$ within each part. Consider (\ref{eq:KL3}) as the definition of the log cut distance. Let $\mathbf{P}$ be a decoded BigClam graph from the affiliation features $\mathbf{F}$.

Denote by $\tilde{\mathbf{P}}$ the matrix with entries 
\[\tilde{p}_{n,m}=-\log(1-p_{n,m})=\mathbf{f}_n^\top\mathbf{f}_m,\] 
and by $\tilde{\mathbf{B}}$ the matrix  with entries $\tilde{b}_{n,m}=-\log(1-b_{n,m})$.
We show that there is no way to make $\| \tilde{\mathbf{P}}-\tilde{\mathbf{Q}}\|_{\square}$ small.

\begin{claim}
    Under the above construction, 
    \[D^0_{\square}(\mathbf{P} || \mathbf{B}) \geq \frac{a^2}{16}.\]
    As a result, BigClam is not a universal autoencoder.
\end{claim}

\begin{proof}
    
Note that $\tilde{b}_{n,m}=a^2$ if $n,m$ are in opposite parts, and $\tilde{b}_{n,m}=0$ if $n,m$ are on the same side.
We index the nodes such that $[N]$ is the first side of the graph, and $[N]+N$ is the second side. For $n\in[N]$ we denote $\mathbf{q}_n=\mathbf{f}_n$ and  $\mathbf{y}_n=\mathbf{f}_{n+N}$.
Next, we use the identity 
\[D^0_{\square}(\mathbf{P} || \mathbf{Q}) = \| \tilde{P}-\tilde{B}\|_{\square },\]
and bound the right-hand-side from below.

First, by the definition of cut norm, for every $\mathcal{U},\mathcal{V}\subset[2N]$,
\[\| \tilde{P}-\tilde{B}\|_{\square} \geq \Big|\frac{1}{4N^2}\sum_{n\in \mathcal{U}}\sum_{m\in \mathcal{V}} (\tilde{p}_{n,m} - \tilde{b}_{n,m})\Big|.\]
Hence, for $\mathcal{U}_1=\mathcal{V}_1=[N]$, $\mathcal{U}_2=\mathcal{V}_2=[N]+1$, $\mathcal{U}_3=[N],\mathcal{V}_3=[N]+N$, and $\mathcal{U}_4=[N]+N,\mathcal{V}_4=[N]$, we have
\[\| \tilde{P}-\tilde{B}\|_{\square} \geq \]
\[\frac{1}{16N^2}\sum_{j=1}^4 \Big|\sum_{n\in \mathcal{U}_j}\sum_{n\in \mathcal{V}_j} (\tilde{p}_{n,m} - \tilde{b}_{n,m})\Big| \]
\[= \frac{1}{16N^2}\sum_{n=1}^N\sum_{m=1}^{N} \mathbf{q}_n^\top \mathbf{q}_m  + \frac{1}{16N^2}\sum_{n=1}^N\sum_{m=1}^{N} \mathbf{y}_n^\top \mathbf{y}_m \]
\begin{equation}
    \label{eq:ClamNpUni0}
    +\frac{1}{8N^2}\sum_{n=1}^N\sum_{m=1}^{N} ( a^2 -\mathbf{q}_n^\top \mathbf{y}_m).
\end{equation}
Denote 
\[\mathbf{q}=\frac{1}{N}\sum_{n=1}^N \mathbf{q}_n, \quad \mathbf{y}=\frac{1}{N}\sum_{n=1}^N \mathbf{y}_n.\]
With these notations, (\ref{eq:ClamNpUni0}) can be written as
\[\| \tilde{P}-\tilde{B}\|_{\square} \geq \]
\[\frac{1}{16}\mathbf{q}^\top \mathbf{q}  + \frac{1}{16}\mathbf{y}^\top \mathbf{y} +\frac{1}{8}( a^2 -\mathbf{q}^\top \mathbf{y})\]
\[=\frac{1}{16}\Big(a^2 + (\mathbf{q}^\top-\mathbf{y}^\top)(\mathbf{q}-\mathbf{y})\Big) \geq \frac{a^2}{16}.\]
\end{proof}

We note that one can similarly show that BigClam is not universal also with respect to the log cut distance of Definition \ref{def:KL2}.

\subsection{BigClam With No Self Loops Approximating Bipartite Graphs}
\label{BigClam With No Self Loops Approximating Bipartites}

Consider the above bipartite graph  $\mathbf{B}$ with $N$ nodes at each part, and probability $1-e^{-a^2}$ for an edge between the two parts, and $0$ within each part. 
If we redefine the BigClam decoder to have no self-loops, namely, $\mathbf{P}$ has entries 
$p_{n,m}=P(n\sim m|\mathbf{f}_n, \mathbf{f}_m)$ for $n\neq m$, and $p_{n,m}=0$ for $n=m$, then one can obtain a bipartite $\mathbf{P}$ with $C=N^2$ communities as follows. 

In the following analysis, an addition or multiplication of a set by a scalar is defined to be the addition or multiplication of every element in the set by this scalar.
Encode each node $n\in[N]$ in part 1 to $\mathbf{f}_n$ with
$f_n^c=a$ for $c\in [N]+n(N-1)$ and $f_n^c=0$ otherwise. Encode every node $n\in [N]+N$ from side 2 to $\mathbf{f}_n$ with
$f_n^c=a$ for $c\in N([N]-1)+n$ and $f_n^c=0$ otherwise. It is easy to see that the corresponding $\mathbf{P}$ is bipartite with edge probability between the parts being $1-e^{-a^2}$.

\subsection{Proof of the Universality of PieClam}
\label{Proof of the Universality of PieClam}

The proof is based on a version of the weak regularity lemma for intersecting communities. While the standard weak regularity lemma \cite{weakReg,Szemeredi_analyst} partitions the graph into disjoint communities, it is well known that allowing the communities to overlap allows using much less communities, which improves the asymptotics of the approximation. The regularity lemma was used in the context of graph machine learning in \cite{Levie_reg_23,Levie_intersecting_reg_24}. To formalize the relevant version of the weak regularity theorem for our analysis, we first need to cite a definition from \cite{Levie_intersecting_reg_24}.

\begin{definition}
    A (hard) intersecting community graph (ICG) with $N$ nodes and  $K$ communities is a matrix $\mathbf{C}\in\mathbb{R}^{N\times N}$ of the following form. There exist $K$ column vectors $\mathbf{Q}=\big(\mathbf{q}_k\in\{0,1\}^N\big))_{k=1}^K\in \{0,1\}^{N\times K}$ and $k$ coefficients $\mathbf{r}=(r_k\in\mathbb{R})_{k=1}^K\in\mathbb{R}^K$ such that
    \[\mathbf{C}= \mathbf{Q}{\rm diag}(\mathbf{r})\mathbf{Q}^{\top }.\]
\end{definition}




The following is a special case of the weak regularity lemma from \cite{Levie_intersecting_reg_24}, up to the small modification to the adjacency matrix, allowing it to have values in $[0,R]$ instead of $[0,1]$. 

\begin{theorem}
\label{thm:reg_sprase}
     Let $\mathbf{A}\in[0,R]^{N\times N}$ be an adjacency matrix of a graph with $N$ nodes. Let $\epsilon>0$. Denote $K=\frac{9R^2}{4\epsilon^2}$. Then, there exists a hard ICG $\mathbf{C}$ with $K$ communities such that 
    \begin{equation}
        \label{eq:reg_sparse}
    \|\mathbf{A}-\mathbf{C}\|_{\square} \leq \epsilon.
    \end{equation}
\end{theorem}

\begin{proof}[Proof of Theorem \ref{thm:IEUniversality}]
 Let $\epsilon>0$. Let $\mathbf{A}\in[0,1]^{N\times N}$ be an adjacency matrix. Let $0<d\leq 1$.

Consider the matrix  $\tilde{\mathbf{A}}$ with entries
\[\tilde{a}_{n,m}=-\log(1-(1-d)a_{n,m}).\]
In the following construction, we build IeClam affiliation features $\mathbf{F}$ and we want
\[1-\exp(-\mathbf{f}_n^{\top}\mathbf{L}\mathbf{f}_m)\approx (1-d)a_{n,m}.\]
Note that $-\log(1-(1-d)a_{n,m})$ is increasing in $(1-d)a_{n,m}$. For $(1-d)a_{n,m}=0$ the value of this function is 0, and for $(1-d)a_{n,m}=1-d$ it is $R=-\log(d)$. Choose $d=\epsilon/2$. For this specific choice of $d$, if we replace $d$ by $\epsilon/2$ in the definition of $D_{\square}$ and omit the infimum, we get an upper bound of $D_{\square}(\mathbf{P},\mathbf{A})$.


Using the overlaping weak regularity lemma, we approximate $\tilde{\mathbf{A}}$ by an ICG with 
\[K=\frac{-9\log(\epsilon/2)^2}{\epsilon^2}\]
communities,
\[\mathbf{C}= \mathbf{Q}{\rm diag}(\mathbf{r})\mathbf{Q}^{\top },\]
such that
\[\|\tilde{\mathbf{A}}-\mathbf{C}\|_{\square} \leq \epsilon/2.\]

Let $\mathbf{r}_+={\rm ReLU}(\mathbf{r})$ and $\mathbf{r}_-={\rm ReLU}(-\mathbf{r})$. Denote
\[\mathbf{C}_+=\mathbf{Q}{\rm diag}(\sqrt{\mathbf{r}_+})\]
and
\[\mathbf{C}_-=\mathbf{Q}{\rm diag}(\sqrt{\mathbf{r}_-}).\]
Denote the rows of $\mathbf{C}_+$ by $\mathbf{t}_n$ and  the rows of $\mathbf{C}_+$ by $\mathbf{s}_n$, for $n=1,\ldots,N$. For each $n\in[N]$ we concatenate $(\mathbf{t}_n,\mathbf{s}_n)$ to define the affiliation feature $\mathbf{f}_n$, with the first $K$ coordinates being the inclusive communities, and the last $K$ coordinates being the exclusive communities. Denote the corresponding IeClam matrix by $\mathbf{P}$.

It is easy to see that $\mathbf{f}_n^\top\mathbf{L}\mathbf{f}_m$ is the $(n,m)$ entry of $\mathbf{Q}{\rm diag}(\mathbf{r})\mathbf{Q}^{\top }$. 
This also proves that
\[\|-\log(1-(1-\epsilon)\mathbf{A})+\log(1-\mathbf{P})\|_{\square}\]
\[=\|\tilde{\mathbf{A}}-\mathbf{C}\|_{\square} \leq \epsilon/2.\] 
%
We can now summarize
\[D_{\square}(\mathbf{P}||\mathbf{A}) \leq \]
\[\epsilon/2+\frac{1}{N^2}\sup_{\mathcal{U},\mathcal{V}\subset[N]}\Big|\log\Big(\prod_{n\in \mathcal{U}}\prod_{m\in \mathcal{V}}
\frac{1-p_{n,m}}{1-(1-\epsilon/2)a_{n,m}}\Big)\Big|\]
\[ =
\epsilon/2+\frac{1}{N^2}\sup_{\mathcal{U},\mathcal{V}\subset[N]}\Big|\sum_{n\in \mathcal{U}}\sum_{m\in \mathcal{V}}\Big(-\log\big(1-(1-\epsilon)a_{n,m}\big) \]
\[\quad \quad \quad \quad\quad \quad \quad \quad + \log\big(1-p_{n,m})\big)\Big)\Big|\]
\[=\epsilon/2 + \| \tilde{\mathbf{P}}-\mathbf{C} \|_{\square} = \epsilon. \]
   
\end{proof}

\subsection{Proof of the Universality of IeClam in the Pairwise Cone of Non-negativity}
\label{Proof of the Universality of IeClam in a Cone of Non-negativity}

For this result, we use the standard weak regularity lemma for non-intersecting classes.  It is based on the weak regularity lemma from \cite{weakReg,Szemeredi_analyst}, see also Lemma 9.3 and Corollary 9.13 from \cite{cut-homo3}.

\begin{definition}
    A \emph{block matrix} $\mathbf{B}$ with $K$ classes is a symmetric matrix $\mathbf{B}\in[0,\infty)^{N\times N}$ for which there exists a partition of $[N]$ into $K$ disjoint sets, called \emph{classes}, $\mathcal{C}_1,\ldots,\mathcal{C}_K$ (with $\cup\mathcal{C}_j=[N]$), such that for every pair of classes $i,j\in[K]$, there is a constant $c_{i,j}\geq 0$ such that  $b_{n,m}=c_{i,j}$ for any two nodes $n\in\mathcal{C}_i$ and $m\in\mathcal{C}_j$.
\end{definition}

\begin{theorem}
\label{thm:reg_sprase2}
     Let $\mathbf{A}\in[0,R]^{N\times N}$ be an adjacency matrix of a graph with $N$ nodes. Let $\epsilon>0$. Denote $K=2^{2\lceil R^2/\epsilon^2\rceil}$.         Then, there exists a block matrix  $\mathbf{B}$ with $K$ (disjoint) classes such that 
    \begin{equation}
        \label{eq:reg_sparse2}
    \|\mathbf{A}-\mathbf{B}\|_{\square} \leq \epsilon.
    \end{equation}
\end{theorem}


We stress that Theorem \ref{thm:reg_sprase2} guarantees non-negative block values of $\mathbf{B}$, while in Theorem \ref{thm:reg_sprase}, in general, the matrix $\mathbf{C}$ may have negative entries.


\begin{proof}[Proof of Theorem \ref{thm:IEUniversality2}]

We start similarly to the proof of Theorem \ref{thm:IEUniversality}. Let $\epsilon>0$. Let $\mathbf{A}\in[0,1]^{N\times N}$ be an adjacency matrix. 
Consider the matrix  $\tilde{\mathbf{A}}\in[0,-\log(\epsilon/2)]$ with entries
\[\tilde{a}_{n,m}=-\log(1-(1-\epsilon/2)a_{n,m}).\]
In the following, we build IeClam affiliation features $\mathbf{F}$  such that
\[1-\exp(-\mathbf{f}_n^{\top}\mathbf{L}\mathbf{f}_m)\approx (1-\epsilon/2)a_{n,m}.\]

By the weak regularity lemma (Theorem \ref{thm:reg_sprase2}), we approximate there is a non-negative block matrix $\mathbf{B}$ with $K=2^{2\lceil -\log(\epsilon/2)^2/\epsilon^2\rceil}$ classes $\mathcal{C}_1,\ldots,\mathcal{C}_K$, such that
\[\|\tilde{\mathbf{A}}-\mathbf{B}\|_{\square} \leq \epsilon.\]

We now take the affiliation space to have $C=K^2$ inclusive communities  and $C=K^2$ exclusive communities.

For each $n\in[N]$, let $k_n$ be the class such that $c\in\mathcal{C}_{n_k}$.  For each pair of classes $i,j\in[K]$, let $c_{i,j}\geq 0$ denote the edge weight between $\mathcal{C}_i$ and $\mathcal{C}_j$.



The feature of each $n\in[N]$ at the inclusive channel $c=(K-1)k_n+k_n$ is $t_n^c=\sqrt{c_{k_n,k_n}}$. It is $t_n^c=\sqrt{c_{k_n,j}/4}$ at inclusive channels $c=(K-1)k_n+j$ and $c=(K-1)j+k_n$, and $s_n^c=-\sqrt{c_{k_n,j}/4}$ at exclusive channels $c=(K-1)k_n+j$ and $c=(K-1)j+k_n$. In all other channels $t^c_n$ and $s^c_n$ are zero. Note that $\mathbf{f}_n$ belongs to the cone of pairwise non-negativity $\mathcal{T}\subset \mathbb{R}^{2K^2}$.

It is now direct to see that  $\mathbf{f}_n\top\mathcal{L}\mathbf{f_m}=c_{k_n,k_m}$.  As a result, as in the proof of Theorem \ref{thm:IEUniversality}, we get
\[D_{\square}(\mathbf{P}||\mathbf{A}) \leq \]
\[\epsilon/2+\frac{1}{N^2}\sup_{\mathcal{U},\mathcal{V}\subset[N]}\Big|\log\Big(\prod_{n\in \mathcal{U}}\prod_{m\in \mathcal{V}}
\frac{1-p_{n,m}}{1-(1-\epsilon/2)a_{n,m}}\Big)\Big|\]
\[ =
\epsilon/2+\frac{1}{N^2}\sup_{\mathcal{U},\mathcal{V}\subset[N]}\Big|\sum_{n\in \mathcal{U}}\sum_{m\in \mathcal{V}}\Big(-\log\big(1-(1-\epsilon)a_{n,m}\big) \]
\[\quad \quad \quad \quad\quad \quad \quad \quad + \log\big(1-p_{n,m})\big)\Big)\Big|\]
\[=\epsilon/2 + \| \tilde{\mathbf{P}}-\mathbf{B} \|_{\square} = \epsilon. \]

  
\end{proof}

\section{Extended Related Work}
\label{Extended Related Work}

\subsection{Message Passing Algorithms and Networks}

The message passing algorithm is a general architecture for processing graph-signals. An MPNN operates on the graph data by aggregating the features in the neighborhood of each node, allowing for information to travel along the edges. The first example of this scheme (Message Passing) was originally suggested by Pearl et al \cite{pearl1982belief},  and was combined with a neural network in \cite{duvenaud2015convolutional}, and later generalized in \cite{gilmer2017neural}.  Most graph neural networks applied in practice are specific instances of MPNNs \cite{MPNN,GCN,GAT}. 
In MPNNs, information is exchanged between nodes along the graph's edges. Each node combines the incoming messages from its neighbors using an \emph{aggregation scheme}, with common methods being summing, averaging, or taking the coordinate-wise maximum of the messages.
Let $T \in \mathbb{N}$ represent the number of layers, and define two sequences of positive integers $(c_t)_{t=0}^T$ and $(d_t)_{t=0}^T$ representing the feature dimensions in the hidden layers $\{\mathcal{G}_t\}_{t=0}^T = \{\mathbb{R}^{c_t}\}_{t=0}^T$ and  $\{\mathcal{S}_t\}_{t=0}^T = \{\mathbb{R}^{d_t}\}_{t=0}^T$. Define the message functions as  $M^t : \mathcal{S}_t \times \mathcal{S}_t \rightarrow \mathcal{G}_t$  and unpdate functions as $U^t: \mathcal{G}_t \times \mathcal{S}_t \rightarrow \mathcal{S}_{t+1}$.
The features $\mathbf{f}_n^{t+1} \in \mathcal{S}_{t+1}$ at layer $t+1$ of the nodes $n\in[N]$ are computed from the features $\mathbf{f}_m^t  \in \mathcal{S}_t$ by 
\begin{align*}
    \mathbf{m}^{t}_n &= \underset{k\in \mathcal{N}(n)}{\sum}{M^t(\mathbf{f}^t_n,\mathbf{f}^t_k)}\\
    \mathbf{f}^{t+1}_n &= U^t(\mathbf{m}^t_n, \mathbf{f}^t_n).
\end{align*}
Here, $M^t$, $U^t$ and $m^t$ are called the message function, update function and mail at time $t$ respectively.  The summation over the messages can also be replaced for any node by any function $Agg_n^t: \prod_{|\mathcal{N}(n)|} \mathcal{G}_t \rightarrow \mathcal{G}_t$
which is permutation invariant. 

At step $T$ a readout space can be defined $\mathcal{R}_T = \mathbb{R}^{b_T}$, with a permutation invariant readout function $R^T$, e.g., summing, averaging, or taking the max of all of the nodes of the graph.
This produces a vector representation of the whole graph.

\subsection{Deep Generative Models}
\label{extended_sec:generative_models}
Generative models in machine learning assume that training data is generated by some underlying probability distribution. One goal in this context is to approximate this distribution, or build a model that approximates a random sampler of data points from this distribution. Hence, generative models can be used to generate synthetic data, mimicking training data by sampling from the distribution 
\cite{harshvardhan2020comprehensive,dinh2016density,mohamed2016learning}, 
or to infer the probability of unseen data by substituting it into the probability function. 
The latter can be useful in tasks like anomaly detection \cite{pang2021anomalynotgraph} in which the model can asses whether a sample is probable under the learned model. 
Two examples of generative models are Generative Adversarial Networks (GANs) \cite{goodfellow2014gan,radford2015dcgan} and Variational Autoencoders (VAEs) \cite{kingma2013vae}, which are used both for inference and generation. 

VAE models consist of an encoder and a decoder. The encoder maps data from a high-dimensional \emph{data space} to a lower-dimensional, simpler, \emph{code space}. The decoder reverses this process, transforming data from the code space back into the data space. The code space serves as a bottleneck, capturing the essential features of the training data, which is often high-dimensional (e.g., images, social network graphs) and thus more complex than the code space. If the model trains by encoding followed by decoding, then minimizing the difference between the input and its decoded version, it is called an \emph{autoencoder}.

In a VAE, training involves encoding each data point to a known distribution (typically Gaussian), sampling from this distribution, decoding it, and then minimizing the difference between the distribution of the original data and the decoded data. For a survey on VAEs, see \cite{tschannen2018vaesurvay}. 

\subsection{Graph Generative Models}
Graph generative models learn a probability distributions of graphs. 
 Such models allow for various tasks 
  where the goal is not only to analyze existing graphs but also to predict or simulate new graph data.

Some classical generative models are pre-defined probabilistic models, e.g., the Erdős–Rényi model, Preferential attachment, Watts–Strogatz model, and more. See \cite{newman2003classical_graph_generative}  for a review. Other graph generative models are learned from data, e.g., \cite{lee2019review, kipf2016variational,you2018graphrnn,bojchevski2018netgan}. 
Applications of generative graph models include social network analysis \cite{wang2018graphgan,grover2019graphite,harshvardhan2020comprehensive}.  anomaly detection \ref{extended_sec:anomaly_detection}, graph synthesis \cite{you2018graphrnn, guo2022systematic}, data augmentation \cite{ding2019data_augmentation}, and protein interaction modeling (e.g. for drug manufacturing) \cite{de2018molgan,ingraham2019generative}, to name a few.  For a review see \cite{ma2021comprehensive}.   

\paragraph{Deep Graph Autoencoders.}
In graph deep learning-based autoencoders,  one estimates the data distribution by learning to embed the nodes to a code space, in which the data distribution is defined to be some standard distribution,  e.g., Gaussian, in such a way that the encoded nodes can be recunstructed back to the graph with small error. Graph VAEs \cite{kipf2016variational,grover2019graphite,JMLR:v21:19-671,mehta2019stochastic} embed the data into the code space by minimizing the evidence lower bound loss comprised of the decoding loss and the KL divergence between the encoded distribution and a Gaussian prior. 

\paragraph{GAN-based Graph generative models.}

In  \cite{wang2018graphgan}, a GAN method for graphs generates a neighborhood for each of the nodes and the discriminator gives a probability score for each edge. This method also formulate a graph version of softmax, and offer a random walk based generating strategy. Another GAN model that is used for anomaly detection is GAAN \cite{chen2020gaan}. In GAAN, the ground truth and generated node attributes are encoded into a latent space from which the adjacency between any two nodes is decoded using a sigmoid of the inner product between their latent features. 

\paragraph{Normalizing Flows-based Graph Models.}
Normalizing flows models \cite{dinh2016density}  also have adaptations for graph data. For example, the work of \cite{liu2019graph,madhawa2019graphnvp} offers a version of coupling blocks that use message passing neural networks.  See more details on normalizing flows in Section \ref{extended_sec:normflows}.

\subsection{Stochastic block models} 
\label{Stochastic block models} 

A stochastic block model (SBM) is a generative model for random graphs. 
A basic stochastic block model is defined by specifying a number of \emph{classes} $K$, the probability $p_k$ of a random node being in block $k$, for $k\in[K]$, where $\sum_kp_k=1$, and an array of values $\mathbf{C}=\{c_{k,l}\}_{k,l=1}^J\in [0,1]^{J\times K}$ indicating edge probabilities between classes. Each node  of a randomly generated graph of $N$ nodes is independently chosen to belong to one of the classes at random, with probabilities  $\{p_k\}_{k\in[K]}$.  Then, the edges of the graph are chosen independently at random according to the following rule. For each $n\in[N]$, denote by $k_n\in[K]$ the class of $n$. Each dyad $(n,m)\in[N]^2$ is chosen to be an edge in probability $c_{k_n,k_m}$. Namely, the entries $a_{n,m}$ of the adjacency matrix of the random graph are independent random Bernoulli variables. See \cite{lee2019review} for a review on SBMs.

For each node $n$, denote by $\mathbf{f}_n\in\{0,1\}^K$ the vector such that $f_n^c=1$ if and only if $k_n=c$. Hence, in a basic SBM, the presence of an edge between nodes $n$ and $m$ follows a Bernoulli distribution with parameter $P(n\sim m|\mathbf{f}_m,\mathbf{f}_m,\mathbf{C}) = {\mathbf{f}_n}^\top  \mathbf{C} \mathbf{f}_m$, where the adjacency matrix is $\mathbf{P} = \mathbf{F}  \mathbf{C} \mathbf{F}^\top$, where $\mathbf{F}\in\mathbb{R}^{N\times K}$ is the matrix where each row $n$ has the feature $\mathbf{f}_n$ \cite{nowicki2001estimation}. 
This model can be extended to intersecting classes, where now $\mathbf{f}_n$ can have more than one nonzero entry \cite{morup2011infinite,miller2009overlappingsbm1,palla2012overlappingsbm2}.

In a \emph{Bernouli-Poisson SBM} \cite{yang2013BigClam,zhou2015bernoullipoisson, shchur2019overlapping}, the probability for a non-edge is modeled by a Poisson distribution. The idea is that the more classes $n$ and $m$ share, the higher the probability that there is an edge between them.  Hence,
\begin{align}
\label{eqn:ber_poi}    
P(n\sim m|\mathbf{f}_m,\mathbf{f}_m,\mathbf{C}) =1 - e^{-{\mathbf{f}_n}^\top  \mathbf{C} \mathbf{f}_m}, 
\end{align}
with the expected number of edges being ${\mathbf{f}_n}^\top  \mathbf{C} \mathbf{f}_m$.

In both the Bernouli and Bernouli-Poisson models, the probabilisic model of the entire graph is given by a product of the probabilities of all of the events
\smallskip
\begin{align*}
\label{eqn:prob_prod}
    & P(E|\mathbf{F},\mathbf{C}) = \\ & \sqrt{\underset{n \in [N]} {\prod}\Big(\underset{m\in \mathcal{N}(n)}{\prod}P(n \sim m)\underset{m \notin \mathcal{N}(n)}{\prod} P(\neg(n \sim m))\Big)}.
\end{align*}
Here, the square root is taken since we assume that the graph is undirected so the product goes over all of the edges twice \cite{aicher2015wsbm,morup2011infinite}. 

When fitting an SBM to a graph, both the class affiliations of nodes and the block structure $\mathbf{C}$ are learned 
\cite{snijders1997sbm_estimation1,nowicki2001sbm_estimation2,latouche2012variational_sbm_estimation}.

\subsection{Community Affiliation Models}

Community detection is a fundamental task in network analysis, aiming to identify groups of nodes that are more densely connected internally than with the rest of the network. This process is useful for understanding the structure and function of complex networks, such as social, biological, and information networks \cite{fortunato2016community}.

Although there exist models in which each node belongs to only one community as in traditional SBM models \cite{holland1983stochastic}, and some relatively new deep learning models such as \cite{cavallari2017learning},
it was shown that real-world networks often exhibit overlapping communities \cite{yang2014structure,yang2013BigClam}, where nodes belong to multiple groups and the probability of connectivity increases the more communities two nodes share. This indeed makes intuitive sense when looking at, e.g., social networks, where the more common interests and social circles people share the more they are likely to connect.  

An example of a community affiliation model that can generate new graphs is AGM \cite{Overlapping2012}, which classifies all of the nodes in a graph into several communities, where each community $k$ has a probability $p_k$ for two member nodes to connect. If we denote by $C_{n,m}$ the set of communities two nodes $n,m \in [N]$  share, then the probability of an edge between $n$ and $m$ is
\[
p(n\sim m) = 1 - \prod_{k \in C_{nm}} (1 - p_k).
\]
To sample a new graph from a trained model, nodes are sampled and assigned communities based on the relative sizes of communities in the training graph, and edges are connected based on their mutual community memberships.  
 
Community affiliations can also be continuous, where nodes have varying degrees of membership in multiple communities.
Community Affiliation models with continuous communities include Mixed Membership SBM (MMSBM)  \cite{airoldi2008mixed} which is an SBM which allows nodes to have mixed memberships in multiple communities, and BigClam \cite{yang2013BigClam}, which is a Bernouli Poisson model model that scales efficiently for sparse graphs.

It is also worth mentioning the NOCD model presented in \cite{shchur2019overlapping} which uses the Bernouli Poisson loss, but embeds the nodes of the graph into the code space using a learned GNN. Another related paper is \cite{sun2019vgraph}, which models the features and communities separately, learning latent features from which it estimates the community affiliation. 

While community affiliations are typically non-negative, there are models where affiliations can be negative. An example is the Signed Stochastic Block Model (SSBM) \cite{jiang2015stochastic}.

\subsection{Anomaly Detection}
\label{extended_sec:anomaly_detection}
Anomaly detection in graphs aims to identify nodes or subgraphs that deviate significantly from the typical patterns within the graph. This is useful in various applications such as network security, fraud detection, and social network analysis. In the unsupervised formulation of the problem, there is no labeled data in the training process, We highlight five works in this direction. GAAN \cite{chen2020gaan} and AEGIS \cite{ding2021aegis} use a generative adversarial approach, training a discriminator to distinguish real and fake nodes. Dominant \cite{ding2019dominant} and AnomalyDAE \cite{fan2020anomalydae} identify anomalies via reconstruction errors of a graph autoencoder. 

Another work is \cite{qiao2024ggad}, which considers a semi-supervised setting. Here, there is a relatively small number of available labeled normal nodes during training (nodes that are known to not be anomalies), and the goal is to predict the anomalies in the unknown nodes.

\subsection{Normalizing Flows}
\label{extended_sec:normflows}

Normalizing flows are deep learning algorithms that estimate  probability distributions, and allow an efficient sampling from this distribution. They do so by constructing an invertible coordinate transformation between the unknown target probability space and a standard probability space with a well known distribution (e.g. Gaussian). This transformation is modeled as a deep neural network, composed of a series of basic transformations called \emph{flows}. 

\paragraph{Notations.}
Let $\mathcal{F}$ represent the space of the target data, with an unknown probability density function $p_\mathcal{F}:\mathcal{F} \rightarrow [0,\infty)$. We denote the elements of $\mathcal{F}$ by $\mathbf{f}$.

Let $\mathcal{Z}$ represent a latent space with a known probability density funcion $p_\mathcal{Z}:\mathcal{Z} \rightarrow [0,\infty)$. We denote the elements of $\mathcal{Z}$ by $\mathbf{z}$. The density function $p_\mathcal{Z}$ is often chosen to be a standard isotropic Gaussian with zero mean and identity covariance, denoted by $\mathcal{N}(0,\mathbf{I})$. 

\paragraph{Goal.} The goal in normalizing flows is to learn an invertible transformation $T_\theta: \mathcal{F} \rightarrow  \mathcal{Z}$ (parameterized by $\theta$) which maps the target space $\mathcal{F}$ to the latent space $\mathcal{Z}$, and preserves probabilities. Namely, $T_{\theta}$ should satisfy: for every measurable subset $F\subset \mathcal{F}$ we have
\[\int_F T_{\theta}(\mathbf{f}) p_{\mathcal{F}}(\mathbf{f})d\mathbf{f} = \int_{T_{\theta}(F)}  p_{\mathcal{Z}}(\mathbf{z})d\mathbf{z}.\]
Here $T_{\theta}(F)= \{z\in\mathcal{Z}\ |\ \exists \mathbf{f}\in F\ {\text{such that}}\ T_{\theta}(\mathbf{f})=\mathbf{z}\}$.
Such a transformation can be seen as a change of variable.

\paragraph{Density Transformation.}
Given an invertible transformation $T_\theta$, the target density $p_\mathcal{F}(\mathbf{f})$ can be expressed in terms of the latent density $p_\mathcal{Z}(\mathbf{z})$. By upholding the constraint that either probability density function has integral 1 over their respective spaces, one can deduce the change of variable formula
\[p_\mathcal{F}(\mathbf{f}) = p_\mathcal{Z}(T_\theta(\mathbf{f})) \bigg| \det \bigg( \frac{\partial T_\theta (\mathbf{f})}{\partial  \mathbf{f}}  \bigg) \bigg|. \]
Here, $\frac{\partial T_\theta (\mathbf{f})}{\partial  \mathbf{f}}$ denotes the Jacobian matrix of the transformation $T_\theta$ with respect to $\mathbf{F}$ and $\det(\cdot)$ denotes its determinant.

\paragraph{Sequential Composition of Flows.}

In practice, the transformation $T_\theta$ is modeled as a composition of $L\in \mathbb{N}$ simpler invertible transformations $\{T_{\theta^i}\}_{i=1}^L$ between the consecutive spaces $\{\mathcal{Z}^i\}$ where $\mathcal{Z}^L = \mathcal{F} $ and $\mathcal{Z}^1 = \mathcal{Z}$ in the previous notations. The invertible mappings $T_{\theta^i}: \mathcal{Z}^i \rightarrow \mathcal{Z}^{i-1}$ are called flows, and we have 
\[T_\theta = T_{\theta^L} \circ T_{\theta^{L-1}} \circ \dots \circ T_{\theta^1}\]
where $\circ$ denotes function composition. The overall density function for $\mathbf{f}$ then becomes
\[p_\mathcal{F}(\mathbf{f}) = p_\mathcal{Z}(T_\theta(\mathbf{f}))\prod_{i=1}^{L} \bigg| \det \bigg( \frac{\partial T_\theta^i (\mathbf{f}_i)}{\partial  \mathbf{f}_i}  \bigg) \bigg|. \]
Here, $\mathbf{f}_i$ is the intermediate representation after applying the first $i-1$ flows.  
The algorithm is optimized by maximum log likelihood, namely, by maximizing
\begin{align*}
     \log\big(p_\mathcal{F}(\mathbf{f})\big) 
     = \log\big(p_\mathcal{Z}(T_\theta(\mathbf{f}))\big) + \sum_{i=1}^{L} \log\Bigg(\bigg| \det \bigg( \frac{\partial T_\theta^i (\mathbf{f}_i)}{\partial  \mathbf{f}_i}  \bigg) \bigg|\Bigg), 
\end{align*}
using gradient descent.

Since $T_\theta$ and every $T_\theta^i$ are invertible, information can also flow in the opposite direction in order to sample data. First, a point $\mathbf{z}$ is sampled $\mathbf{z} \sim p_\mathcal{Z}$ and the inverse transformation $T_\theta^{-1}$ is applied to map the sample $\mathbf{z}$ into the data space $\mathcal{F}$. Since $T$ is composed of a series of flows, the generated point in the data space is 
\[\mathbf{f} = T_\theta^{-1}(\mathbf{z}) = T_{\theta^1}^{-1} \circ T_{\theta^2}^{-1} \circ \dots \circ T_{\theta^L}^{-1}(\mathbf{z}).\]

\paragraph{RealNVP.}
One popular method of implementing normalizing flow is the \emph{Real Valued Non-Volume Preserving (RealNVP)}. In this model, the flows $T_{\theta^i}$ are modeled as mappings called \emph{coupling blocks}. In a coupling block, the input  vector $\mathbf{f}^i$ (or $\mathbf{z}^i$, in the generative direction) is split into two vectors: $\mathbf{f}^i_A$ and $\mathbf{f}^i_B$, each the same dimension. Namely, $\mathbf{f}^i=(\mathbf{f}^i_A,\mathbf{f}^i_B)$.
The flow $T_{\theta^i}$ is then defined to be
\begin{align}
    \mathbf{f}_A^{i-1}& = \mathbf{f}_A^i \\
    \mathbf{f}_B^{i-1} &= \mathbf{f}_B^i \odot \mathbf{s}_{\theta^i}(\mathbf{f}_A^i) + \mathbf{t}_{\theta^i}(\mathbf{f}_A^i) 
\end{align}
Where $(\mathbf{s}_{\theta^i}(\cdot), \mathbf{t}_{\theta^i}(\cdot))$ are Multi Layer Perceptrons (MLPs) parameterized by $\theta^i$, and $\odot$ is elementwise product. 
In addition, the elements of $\mathbf{f}^i$ are permuted at every $i$ before applying the coupling block, using predefined permutations, that are parts of the hyperparameters of the model.
It is easy to see that the Jacobian $\bigg( \frac{\partial T_\theta^i (\mathbf{f}_i)}{\partial  \mathbf{f}_i}  \bigg)$ is a diagonal matrix with $1$ for the first $\dim(\mathbf{f}_A^i)$ elements, and $s_{\theta^i}(\mathbf{f}_A^i)$ for the last $\dim(\mathbf{f}_B^i)$ elements. This  makes the log of the determinant of the Jacobian
\begin{equation}
\log\Bigg(\bigg| \det \bigg( \frac{\partial T_\theta^i (\mathbf{f}_i)}{\partial  \mathbf{f}_i}  \bigg) \bigg|\Bigg) = \sum_{j=0}^{\dim(\mathbf{f}_B^i)}\log \Big( \mathbf{s}_{\theta^i}(\mathbf{f}_A^i)_j \Big). 
 \label{eqn_ext:log_det_jac}   
\end{equation}
Hence, this determinant is easily computed in practice.

For generation, the inverse transformation can be calculated easily as
\begin{align*}
\mathbf{z}_A^{i+1} &= \mathbf{z}_A^{i} \\
\mathbf{z}_B^{i+1} &= \frac{\mathbf{z}_B^{i} - \mathbf{t}(\mathbf{z}_A^i)}{\mathbf{s}(\mathbf{z}_A^i)}
\end{align*}
where the division is elementwise.
Here, the log determinant of the Jacobian can be calculated similarly to (\ref{eqn_ext:log_det_jac}),  applied to $1/\mathbf{s}(\mathbf{z}_A^i)$.

\subsection{The Lorentz Inner Product}
\label{The Lorentz Inner Product}

The Lorentz inner product is a bilinear form used in the context of special relativity to describe the spacetime structure. In a four-dimensional spacetime, the Lorentz inner product between two vectors \(\mathbf{v}\) and \(\mathbf{w}\) is given by:

\[
\langle \mathbf{v}, \mathbf{w} \rangle = v_0w_0 - v_1w_1 - v_2w_2 - v_3w_3
\]
Here, \(v_0\) and \(w_0\) represent the "time" components, while \(v_1, v_2, v_3\) and \(w_1, w_2, w_3\) represent the "spatial" components. The negative sign in front of the time component is what distinguishes the Lorentz inner product from the standard Euclidean inner product, making it suitable for modeling the geometry of spacetime where time and space are treated differently. Note that due to the subtraction, the Lorenz inner product is in fact not an inner product as it is not positive definite. In the context of special relativity, the points for which the Lorenz product remain positive define the so called "light cone" structure of spacetime, separating events into those that are causally connected and those that are not. 

The Lorentz inner product is a specific example of a broader class of inner products known as \textbf{pseudo-Euclidean inner products}. In a pseudo-Euclidean space, the inner product can have a mixture of positive and negative signs, leading to different geometric properties. These spaces generalize the concept of Euclidean space by allowing for non-positive definite metrics.

\section{PClam and PieClam as Graphons}

A graphon is a model which can be seen as a graph generative model that extends SBMs.
A graphon \cite{graphon1,cut-homo3} can be seen as a weighted graph with a ``continuous'' node set $[0,1]$.

\begin{definition}
    The space of graphons $\mathcal{W}$ is defined to be the set of all measurable function  $W:[0,1]^2\rightarrow [0,1]$ which are symmetric, namely $W(x,y)=W(y,x)$. 
\end{definition}
The edge weight $W(x,y)$ of a graphon $W\in \mathcal{W}$ can be seen as the probability of having an edge between node $x$ and node $y$.   Given a graphon $W$, a random graph is generated by sampling  independent uniform nodes $\{X_n\}$ from the graphon domain $[0,1]$, and connecting each pair $X_n,X_m$ in probability $W(X_n,X_m)$ to obtain the edges of the graph.

We next show that Clam models with prior are special cases of graphon models. 
Note that the prior $p$ defines a standard atomless probability spaces over the community affiliation space $\mathbb{R}^K$. Since all standard atomeless probability spaces are equivalent, there is a probability preserving a.e. bijection  $\xi_p:[0,1]\rightarrow\mathbb{R}^K$ that maps the prior probability to the uniform probability over $[0,1]$. Now, the PieClam model for generating a graph of $N$ nodes can be written as follows.
\begin{itemize}
    \item 
    Sample $N$ points $\{X_n\}_{n=1}^N\subset[0,1]$ uniformly. Observe that $\{\xi_p(X_n)\}_{n=1}^N\subset \mathbb{R}^K$ are indepndent samples via the probability density $p$.
    \item 
    Connect the points according to the BigClam of IeClam models $P(n\sim m| \xi_p(X_n),\xi_p(X_m))$ (\ref{eq:ClamEdge}) or (\ref{eq:LorFeature}).
\end{itemize}
This shows that PClam and PieClam coincide with the generative graphon model $W(x,y)=P(n\sim m| \xi_p(X),\xi_p(Y))$ where $P$ is defined either by  (\ref{eq:ClamEdge}) or by (\ref{eq:LorFeature}).

\section{Extended Details on Experiments}
\label{Extended Details on Experiments}

All of the experiments were run on Nvidia GeForce RTX 4090 and Nvidia L40 GPUs.
Our code can be found in: 

\url{https://anonymous.4open.science/r/PieClam-49C4}

\subsection{Additional Architecture Details in PieClam and PClam}

\paragraph{Added Affiliation Noise.}
When training the prior, overfitting may cause the probability to spike around certain areas in the affiliation space, e.g., around the affiliation features of the nodes. To avoid this issue, we add gaussian noise to the affiliation vector as a regularization at each step of training the normalizing flow prior model. The normalizing flow model then transforms the same point to a slightly different location in the code space. The noise optimization therefore provides a resolution to the prior, smoothing the distribution in the affiliation space.  Noise addition is the primary regularization method we have used when training the prior in all of our experiments.

\paragraph{Densification.}
For very sparse datasets, Clam models may not behave well. In such cases, the community structure is sometimes unstable, where the same node can find itself in different communities based on slightly different initial conditions. In order to strengthen the connections within communities, we apply a two-hop densification scheme on very sparse graphs. Namely, we connect two disjoint nodes $n,m$ with an edge if there is a third node $k$ for which $n \sim k$ and $m\sim k$. We find that this scheme improves anomaly detection on Elliptic, Photo and Reddit datasets. 
Densification may strengthen the community structure in some datasets, but can also destroy it in others. In the experiments on synthetic datasets we did not use densification, as these graphs are relatively dense. 

\subsection{Anomaly detection}

\paragraph{Metric.} In order to measure the accuracy of the anomaly detection classification we use the area under the Receiver Operating Characteristic (ROC) curve \cite{peterson1954roc}. The curve is a plot of the true positive rate (TPR) against the false positive rate (FPR) for the range of the threshold values between the value that classifies all samples as true and the value that classifies all as false. The area under the curve signifies the general tendency of the curve toward the point (0,1) for which FPR$= 0$ and TPR$= 1$.

\paragraph{Additional experiment.} We add additional experiments on anomaly detection using PieClam, where we averaged the accuracy over 10 rounds and included error bars.

In this setup we compose three iteration steps of the form $\mathbf{F}$-500 $\rightarrow$ $p$-1300 with initial learning rates of $3e-6$ for the features and $2e-6$ for the prior, and initial noise amplitude of $0.05$. The learning rates and noise amplitude were halved after every $\mathbf{F} \rightarrow p$ iteration.  The results are in Table~\ref{Table2}.

Furthermore, in this section we also run our anomaly detection methods without using the node features, namely, only using the graph structure. We find that we comparable results on Photo and Reddit with or without node features, but that the detection on Elliptic decreases significantly. Still, anomaly detection on Elliptic without node features is competitive with state of the art models that do use node features.  The results are in Table~\ref{Table22}.

\begin{table*}[t]
\centering
\begin{tabular}{|c|c|c|c|}
    \hline
  Method  & Reddit& Elliptic& Photo\\
    \hline

 (S)- IeClam& \textbf{0.639}& 0.440&\textbf{0.592}\\
 (S) - PieClam& 0.6320 ± 0.0054&  
0.4354 ± 0.0005&*0.5727 ± 0.0151\\
 (P) - PieClam& 0.5089 ± 0.0219& \textbf{0.6201 ± 0.0176}&0.4252 ± 0.0074\\
 (PS) - PieClam& *0.6329 ± 0.0052& \underline{0.5294 ± 0.0089}&0.5289 ± 0.0127\\
 (S) - BigClam & \underline{0.637}& 0.434&\underline{0.581}\\
 \hline
 DOMINANT  & 0.511& 0.296&0.514\\
 AnomalyDAE &0.509& *0.496&0.507\\
 OCGNN  &0.525 & 0.258&0.531\\
 AEGIS &0.535 & 0.455&0.552\\
 GAAN &0.522 &0.259 &0.430\\
 TAM &0.606 & 0.404 & 0.568\\
 \hline
\end{tabular}
\caption{Comparison of Clam anomaly detectors with competing methods. 
 First place in \textbf{boldface}, second with \underline{underline}, third with *star. We observe that our methods are first place on all datasets. Moreover, S- IeClam, PS-PieClam and S BigClam each beats the competing methods in two out of the three datasets . The accuracy metric is areas under curve (AUC).}
\label{Table2}
\end{table*}
We proceed to compare the optimization that includes the node features to the optimization that didn't. The comparison is presented in

\begin{table*}[t]
\centering
\begin{tabular}{|c|c|c|c|}
    \hline
  Method  & Reddit & Elliptic & Photo \\
    \hline
  (S) - PieClam & 0.6320 ± 0.0054 & 0.4354 ± 0.0005 & \textbf{0.5727 ± 0.0151} \\
  (P) - PieClam & 0.5089 ± 0.0219 & \textbf{0.6201 ± 0.0176} & 0.4252 ± 0.0074 \\
  (PS) - PieClam & *0.6329 ± 0.0052 & \underline{0.5294 ± 0.0089} & 0.5289 ± 0.0127 \\
    \hline
  (S) - PieClam (No Attr) & \underline{0.6346 ± 0.0015} & 0.4349 ± 0.0009 & *0.5696 ± 0.0065 \\
  (P) - PieClam (No Attr) & 0.4530 ± 0.0142 & *0.4826 ± 0.0134 & 0.4977 ± 0.0747 \\
  (PS) - PieClam (No Attr) & \textbf{0.6351 ± 0.0015} & 0.4349 ± 0.0009 & \underline{0.5697 ± 0.0065} \\
    \hline
\end{tabular}
\caption{Comparison of attributed and unattributed PieClam anomaly detection.}
\label{Table22}
\end{table*}

\subsubsection{T-Test Comparison Summary.}
We conducted t-tests to determine if node features/attributes ("With Features") significantly improve performance compared to not using attributes ("Without Features"), across the Reddit, Elliptic, and Photo datasets for each method (S, P, PS). Even though S method is not affected directly by the affiliation vectors, the latter can affect the convergence of the prior in the affiliation space and have an indirect effect.
\begin{itemize}
    \item \textbf{Reddit Dataset}: Attributes did not significantly improve performance for the "Star" (S) and "Prior Star" (PS) methods. However, "Prior" (P) did show a significant improvement with attributes (p-value=0.0002).
    \item \textbf{Elliptic Dataset}: "Prior" (P) and "Prior Star" (PS) both showed significant improvements with attributes (p-value$<$0.0001). However, the "Star" (S) method did not significantly benefit from attributes.
    \item \textbf{Photo Dataset}: The "Star" (S) method did not show significant difference (p-value=0.6864). The "Prior Star" (PS) method, however, did show a significant benefit from attributes (p-value$<$0.0001), and "Prior" (P) had a borderline significant result (p-value=0.0529).
\end{itemize}

In conclusion, attributes did not make a difference for the "Star" (S) method across all datasets (as would be expected), and for the "Prior Star" (PS) method in the Reddit dataset. However, the "Prior" (P) and "Prior Star" (PS) methods mostly showed improvement in the Elliptic dataset. 

\subsection{SBM Reconstruction And Distance Convergence}
We extend the experiment in Section \ref{sec:reconstruction_sbm} by considering another  synthetic SBM, which is off-diagonal dominant. As in Section \ref{sec:reconstruction_sbm}, we sample a simple graph with $N=210$ nodes from the SBM. The non zero probability blocks have a probability of $0.5$. We estimate the log cut distance between the Clam models (BigClam and IeClam) and the SBM during training. We consider for both methods affiliation spaces of dimensions $2,4$ and $6$. In addition, we calculate the cut distance and l2 errors between the Clam models and the SBM. The results are shown in Figures~\ref{fig:sbm_pclam1}--\ref{fig:sbm_pclam6}.


\begin{figure}[H]
\centering
    \begin{minipage}{0.3\columnwidth}
\includegraphics[width=\textwidth]{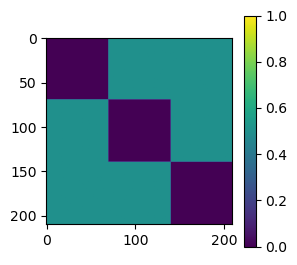}       
\end{minipage}\hspace{0.02\columnwidth}  
\begin{minipage}{0.3\columnwidth}
\includegraphics[width=\textwidth]{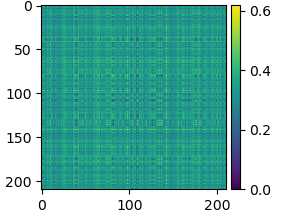}       \end{minipage}\hfill

\begin{minipage}{1.0\columnwidth}
\includegraphics[width=\textwidth]{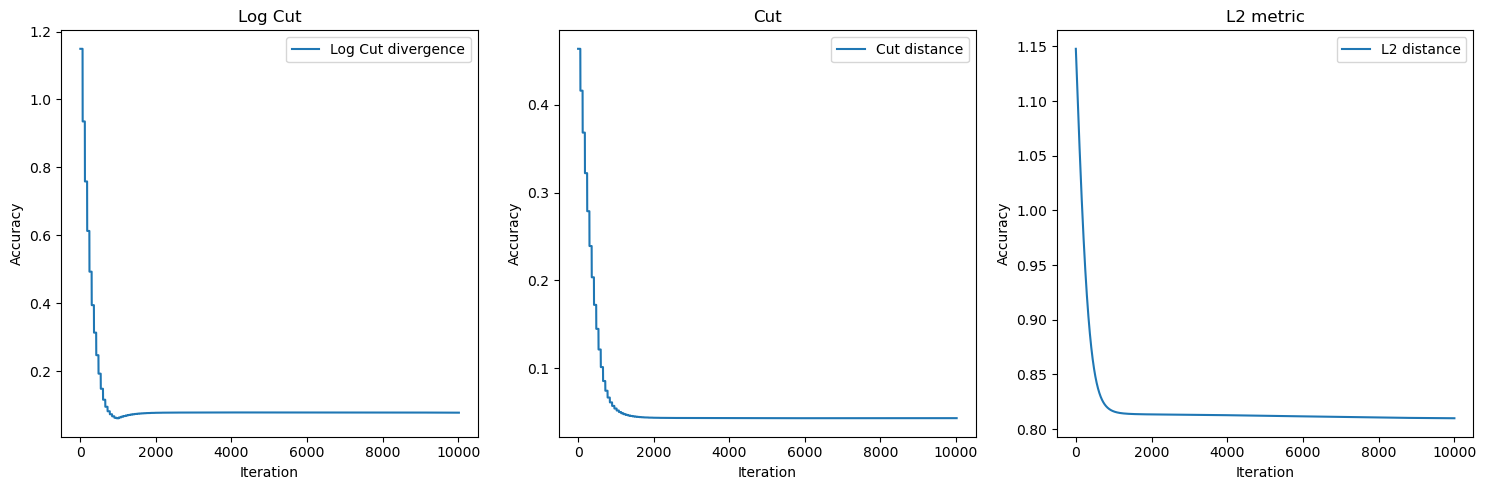}
       \end{minipage}
    \caption{Left to right: Target SBM. Fitted BigClam graph with two communities. Error as a function of optimization iteration where error is, left to right,  log cut distance, cut distance, l2 distance. After convergence, the log cut distance between the SBM and BigClam is 0.0776.}
    \label{fig:sbm_pclam1}
\end{figure}

\begin{figure}[H]
\centering
    \begin{minipage}{0.3\columnwidth}
\includegraphics[width=\textwidth]{images/sbm_reconstruction/halfdiag.png}       
\end{minipage}\hspace{0.02\columnwidth}  
\begin{minipage}{0.3\columnwidth}
\includegraphics[width=\textwidth]{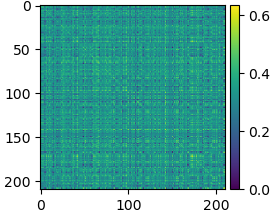}       \end{minipage}\hfill  
\begin{minipage}{1.0\columnwidth}
\includegraphics[width=\textwidth]{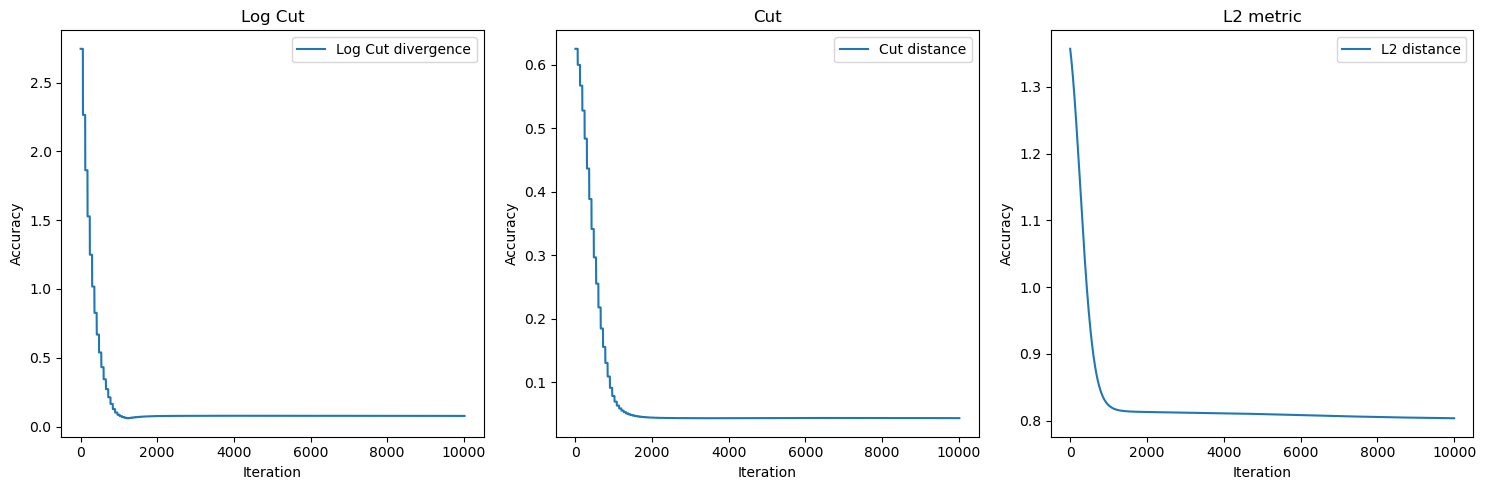}
       \end{minipage}
    \caption{Left to right: Target SBM. Fitted BigClam graph with four communities. Error as a function of optimization iteration where error is, left to right,  log cut distance, cut distance, l2 distance. After convergence, the log cut distance between the SBM and BigClam is 0.0775.}
    \label{fig:sbm_pclam2}
\end{figure}

\begin{figure}[H]
\centering
    \begin{minipage}{0.3\columnwidth}
\includegraphics[width=\textwidth]{images/sbm_reconstruction/halfdiag.png}       
\end{minipage}\hspace{0.02\columnwidth}    
\begin{minipage}{0.3\columnwidth}
\includegraphics[width=\textwidth]{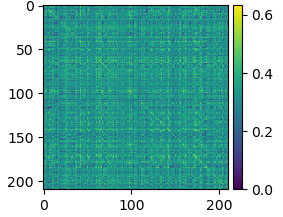}       \end{minipage}\hfill  
\begin{minipage}{1.0\columnwidth}
\includegraphics[width=\textwidth]{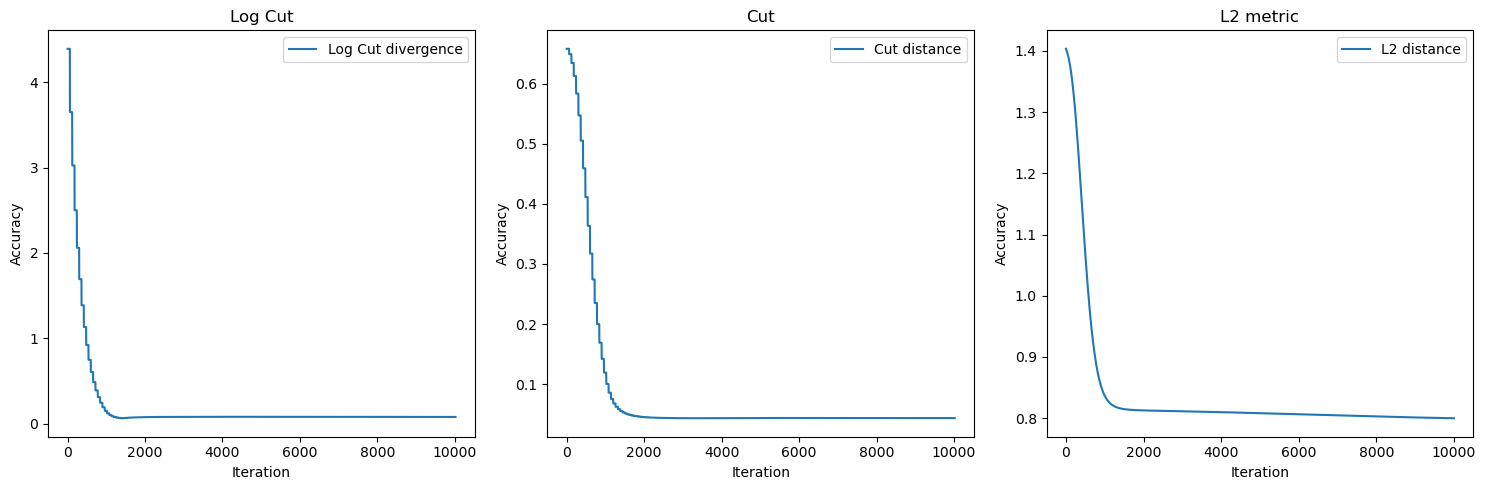}
       \end{minipage}
    \caption{Left to right: Target SBM. Fitted BigClam graph with six communities. Error as a function of optimization iteration where error is, left to right,  log cut distance, cut distance, l2 distance. After convergence, the log cut distance between the SBM and BigClam is 0.0776.}
    \label{fig:sbm_pclam3}
\end{figure}


\begin{figure}[H]
\centering
    \begin{minipage}{0.3\columnwidth}
\includegraphics[width=\textwidth]{images/sbm_reconstruction/halfdiag.png}       
\end{minipage}\hspace{0.02\columnwidth}    
\begin{minipage}{0.3\columnwidth}
\includegraphics[width=\textwidth]{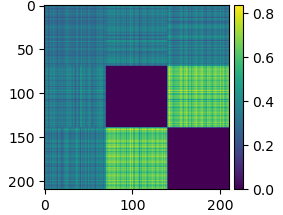}       \end{minipage}\hfill 

\begin{minipage}{1.0\columnwidth}
\includegraphics[width=\textwidth]{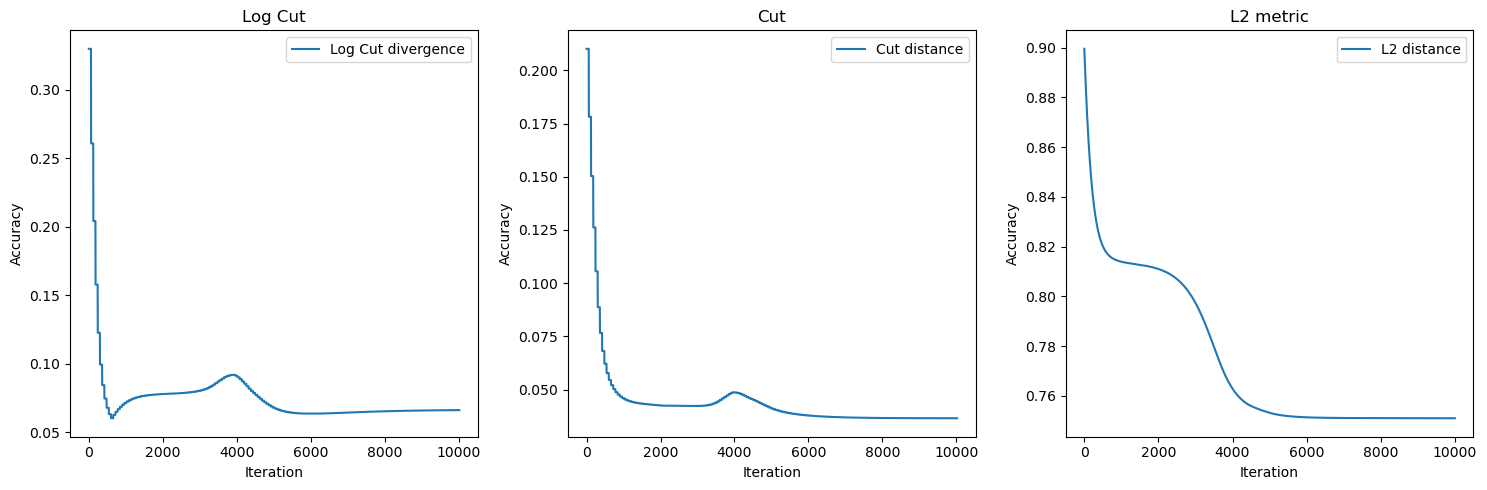}
       \end{minipage}
    \caption{Left to right: Target SBM. Fitted IeClam graph with two communities. Error as a function of optimization iteration where error is, left to right,  log cut distance, cut distance, l2 distance. After convergence, the log cut distance between the SBM and IeClam is 0.0662.}
    \label{fig:sbm_pclam4}
\end{figure}

\begin{figure}[H]
\centering
    \begin{minipage}{0.3\columnwidth}
\includegraphics[width=\textwidth]{images/sbm_reconstruction/halfdiag.png}       
\end{minipage}\hspace{0.02\columnwidth}    
\begin{minipage}{0.3\columnwidth}
\includegraphics[width=\textwidth]{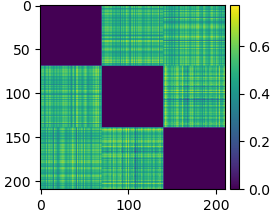}       \end{minipage}\hfill

\begin{minipage}{1.0\columnwidth}
\includegraphics[width=\textwidth]{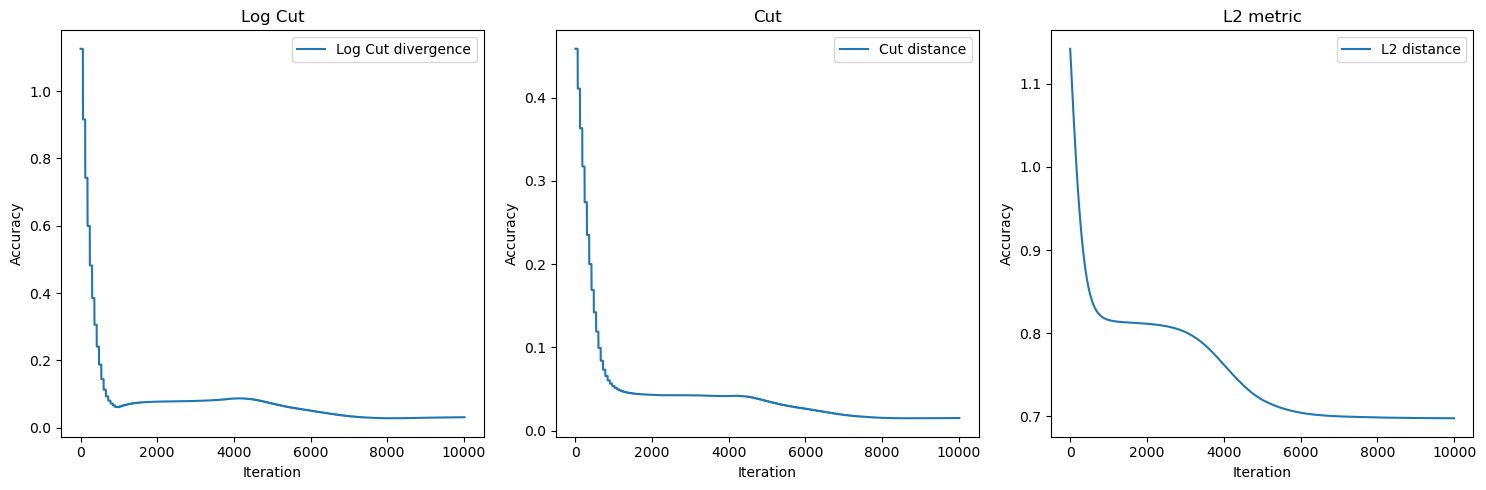}
       \end{minipage}
    \caption{Left to right: Target SBM. Fitted IeClam graph with four communities. Error as a function of optimization iteration where error is, left to right,  log cut distance, cut distance, l2 distance. After convergence, the log cut distance between the SBM and IeClam is 0.0312.}
    \label{fig:sbm_pclam5}
\end{figure}

\begin{figure}[H]
\centering
    \begin{minipage}{0.3\columnwidth}
\includegraphics[width=\textwidth]{images/sbm_reconstruction/halfdiag.png}       
\end{minipage}\hspace{0.02\columnwidth}    
\begin{minipage}{0.3\columnwidth}
\includegraphics[width=\textwidth]{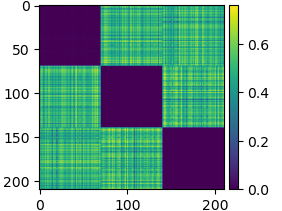}       \end{minipage}\hfill

\begin{minipage}{1.0\columnwidth}
\includegraphics[width=\textwidth]{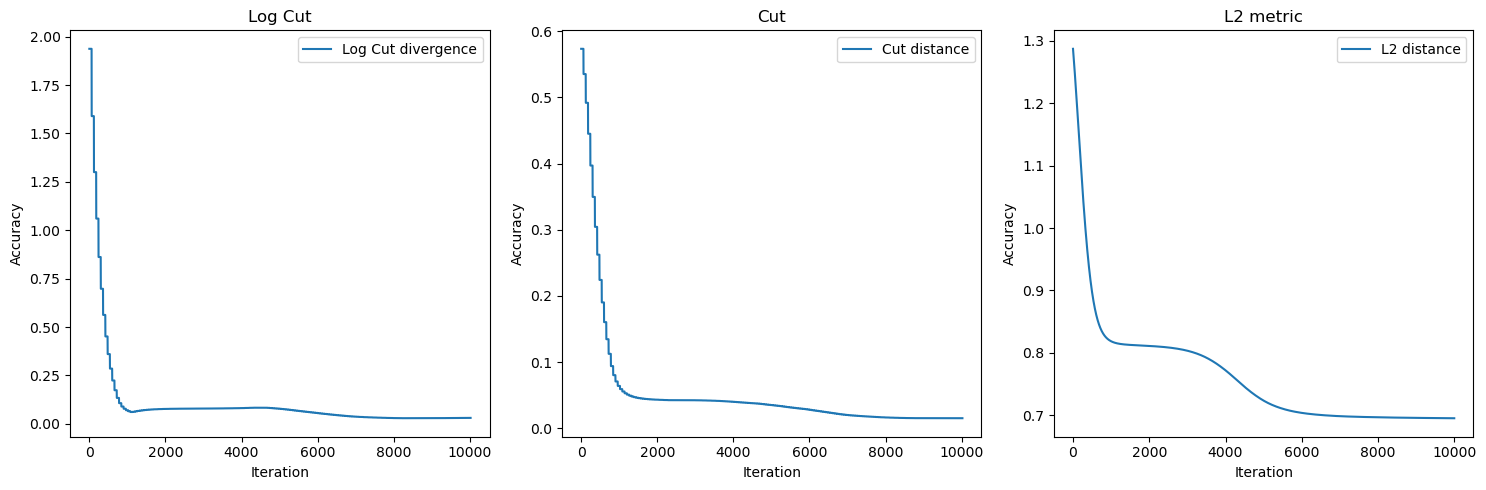}
       \end{minipage}
    \caption{Left to right: Target SBM. Fitted IeClam graph with six communities. Error as a function of optimization iteration where error is, left to right,  log cut distance, cut distance, l2 distance. After convergence, the log cut distance between the SBM and IeClam is 0.0309.}
    \label{fig:sbm_pclam6}
\end{figure}







\end{document}